\documentclass{article}
\usepackage{graphicx} 

\usepackage{comment}

\usepackage[preprint]{arxiv_2025}
\usepackage[utf8]{inputenc} 
\usepackage[T1]{fontenc}    
\usepackage{hyperref}       
\usepackage{url}            
\usepackage{booktabs}       
\usepackage{amsfonts}       
\usepackage{nicefrac}       
\usepackage{microtype}      
\usepackage{xcolor}         

\usepackage{graphicx}
\usepackage{caption}
\usepackage{subcaption}
\usepackage{booktabs}
\usepackage{algorithm}
\usepackage[algo2e]{algorithm2e}
\RestyleAlgo{ruled}
\usepackage{natbib}

\usepackage{amsmath}
\usepackage{amssymb}
\usepackage{amsthm}
\usepackage{dsfont}
\usepackage{mathtools}

\theoremstyle{plain}
\newtheorem{theorem}{Theorem}[section]
\newtheorem{proposition}[theorem]{Proposition}
\newtheorem{lemma}[theorem]{Lemma}
\newtheorem{corollary}[theorem]{Corollary}
\theoremstyle{definition}
\newtheorem{definition}[theorem]{Definition}

\theoremstyle{remark}

\usepackage{shortcuts}

\title{Data Depth as a Risk}

\author{%
  Arturo Castellanos \\
  LTCI, \\
  Telecom Paris, Institut Polytechnique de Paris, France\\
  91120 Palaiseau \\
  \texttt{arturo.castellanos@telecom-paris.fr} \\
   \And
  Pavlo Mozharovskyi \\
  LTCI, \\
  Telecom Paris, Institut Polytechnique de Paris\\
  91120 Palaiseau \\
  \texttt{pavlo.mozharovskyi@telecom-paris.fr} \\
}

\begin{document}

\maketitle

\begin{abstract}
Data depths are score functions that quantify in an unsupervised fashion how central is a point inside a distribution, with numerous applications such as anomaly detection, multivariate or functional data analysis, arising across various fields. The halfspace depth was the first depth to aim at generalising the notion of quantile beyond the univariate case.
Among the existing variety of depth definitions, it remains one of the most used notions of data depth. Taking a different angle from the quantile point of view, we show that the halfspace depth can also be regarded as the minimum loss of a set of classifiers for a specific labelling of the points. By changing the loss or the set of classifiers considered, this new angle naturally leads to a family of "loss depths", extending to well-studied classifiers such as, e.g., SVM or logistic regression, among others. This framework directly inherits computational efficiency of existing machine learning algorithms as well as their fast statistical convergence rates, and opens the data depth realm to the high-dimensional setting. Furthermore, the new loss depths highlight a connection between the dataset and the right amount of complexity or simplicity of the classifiers. The simplicity of classifiers as well as the interpretation as a risk makes our new kind of data depth easy to explain, yet efficient for anomaly detection, as is shown by experiments.
\end{abstract}

\section{Introduction}

Data depths are statistical functions that measure the degree of centrality of a point inside a distribution by extending notions of quantile to the multivariate setting. They have many applications such as missing data imputation~\cite{Mozharovskyi20imput}, statistical tests~\citep{Liu93}, anomaly detection~\citep{mozharovskyi2022anomaly}, across various fields such as NLP~\citep{Colombo2022}, adversarial attacks~\citep{picot2023a} and uncertainty quantification with neural networks~\citep{gamboa24}. Contrary to one-class novelty detection algorithms such as One-Class SVM~\citep{ocsvm}, Deep SVDD~\citep{ruff18a}, One-Class GAN~\citep{ocgan} and many others~\citep{Wu21}, data depths do not learn a generic function on some training data to apply universally on test data, but aim at being robustness by defining an optimisation program specific to each data point, and therefore withstand contamination in the training data. Among various data depths, the halfspace depth~\citep{tukey1975mathematics} has been one of the most studied: it looks at quantiles after projecting the multivariate data along some one-dimensional directions, and computes some infimum over all possible directions in $\bbR^d$. Some price to pay for this robustness is a computational time that typically does not scale well with the dimension. Another drawback is that depths such as halfspace depth often assumes the distribution to have one mode and of convex shape, although those two drawbacks have been tackled by some extension of halfspace depth~\citep{castellanos2023fast} using gradient descent for optimisation and kernel methods for non-convex multimodality.

On the other hand, in the last decade, machine learning techniques have shown impressive results, in particular in the supervised setting, with fast optimisation algorithms working in huge dimensions.
An inspiring connection between Support Vector Machine (SVM) classification and the Maximum Mean Discrepancy (MMD~\citep{gretton12}) distance between positively and negatively labelled data have already been shown in~\citet{reid11}. Following this intuition, we build a general framework we call \emph{loss depths}, by embedding the data point of interest as a Dirac measure on it (negatively labelled), and try to separate it from the distribution (positively labelled), connecting therefore to supervised classification. By evaluating the loss of the algorithm, we quantify how easy or not it is separate the point from the data, and so define its degree of centrality. In section~\ref{sec:HD}, we recover the halfspace depth with this framework, which we generalise in section~\ref{sec:extension}.
To avoid the collapsing of the depth to zero, the classifiers should be \textbf{\emph{really} ``artificially'' intelligent} and not be strong enough to separate the ``median'' points. This gives us incentive to use simple models such as SVM or logistic regression whose associated loss depths we analyse further in section~\ref{sec:linear}. Simplicity of the model let us analyse their statistical convergence rates which are dimension-independent. 
Simplicity also gives an easy interpretation of the decision as shown in experiments in section~\ref{sec:XP}, which suggest that shallow architectures can be a way to get the full picture of some datasets.

\paragraph{Notations}
We will consider loss functions for classification of the shape $l: \bbR  \times \{-1,1\} \to \bbR+$. For such loss $l$, we will define $l_+ : \bbR, \hat{y} \mapsto l(\hat{y},+1)$ the loss on a positive label and $l_- : \bbR, \hat{y} \mapsto l(\hat{y},-1)$ its negative counterpart. Denote $(\cX, c)$ a metric space, $\cP(\cX)$ the space of probability distributions over $\cX$, and $W_c$ the 1-Wasserstein distance with cost $c$ 
between such distributions~\citep{villani2009optimal}:
\begin{equation*}
    W_c(\mu,\nu) = \inf_{\pi \in \Pi(\mu,\nu)} \iint c(x,y) \mathrm d\pi(x,y)
\end{equation*}
where $\Pi(\mu,\nu)$ denotes all the possible couplings between $\mu$ and $\nu$.
We call \emph{Gaussian kernel} (over $\bbR^d \times \bbR^d$) the kernel function: $k(x,y)=e^{-\gamma||x-y||^2}$, $\gamma>0$, see also background in Appendix~\ref{appendsec:RKHS}.

\section{Halfspace Depth: another interpretation}\label{sec:HD}
A usual interpretation of halfspace depth is a generalisation of a survival function to the multivariate settings. For a given $z$ in dataspace $\cX$ and a given distribution $Q$ over $\cX$: 
 \begin{align}
        \TukeyD(z|Q) = \inf_{u \in \sphere} \mathbb{P}(\langle u, X\rangle \geq \langle u, z\rangle) \text{ for } X\sim Q \label{generalTukey}
    \end{align}
where $\sphere$ is the unit sphere of $\mathbb{R}^d$.
To sketch the idea when the infimum is a minimiser: the halfspace depth computes the survival function after projecting in one dimension along some direction, and this direction is preferably picked among all possible as the one that would give minimum value.
Here, we suggest a new interpretation of the halfspace depth more related to the supervised machine learning mindset: as a risk. 
Indeed, we can rewrite it as the best $0\text{-}1$ loss attainable among linear classifiers on 
some particular distribution $\LaDi$:

\begin{definition} For a probability distribution $Q$ over some domain $\cX$, we consider the following (artificially) \emph{\textbf{labelled distribution}} over $\cX \times \{-1,1\}$, where $\cX$ is the data space (usually $\bbR^d$) and $\{-1,1\}$ the label space:
    $\LaDi = {\color{red} \frac{1}{2}Q \otimes \delta_1} +{\color{blue}\frac{1}{2} \delta_z \otimes \delta_{-1}}$.
\end{definition}
In other words if $(x,y)\sim \LaDi$, half of the time $x$ is distributed according to $Q$, with a {\color{red}positive} label $y=1$, and half of the time it is drawn as $z$, with a {\color{blue}negative} label $y=-1$.

\begin{proposition}
 \begin{align}
        \TukeyD(z|Q) = \inf_{f \in Lin^*} \bbE_{(x,y)\sim \LaDi} 2 L_{0\text{-}1}(f(x),y) \label{lossTukey}
    \end{align}    
    where 
    $Lin^* = \{ f:\bbR^d \to \bbR, x \to \langle w,x \rangle + b | w\in \bbR^d\setminus \{0\},  b\in \bbR\}$ is the space of (non-null) linear classifiers, and $L_{0\text{-}1}$ is the zero-one loss defined by $L_{0\text{-}1}(f(x),y) = \mathds 1_{[yf(x)<0]}$.
\end{proposition}

\begin{proof}
    Since the point $z$ represents at least a mass of one half of the distribution, misclassifying it implies that the loss $L_{0\text{-}1}$ is at least $1/2$ while classifying it correctly ensures the loss is no more than $1/2$, so we can assume we restrict ourselves to classifiers that classifies $z$ correctly (see Figure~\ref{subfig:HD1}). For any classifier $w,b$ and the associated halfspace $\bbH^-_{w,b}=\{ x |\langle w,x \rangle + b \leq 0\}$ such that $z\in\bbH^-_{w,b}$, notice that the amount of positively-labelled mass on the negatively-labelled halfspace is potentially reduced when $b$ is shifted to $-\langle w,z \rangle$ such that $z$ is exactly on the boundary hyperplane: if $z\in\bbH^-_{w,b}$, $Q(\bbH^-_{w,-\langle w,z \rangle}) \leq Q(\bbH^-_{w,b})$ as $\bbH^-_{w,-\langle w,z \rangle} \subseteq\bbH^-_{w,b}$ (see Figure~\ref{subfig:HD2}). Therefore we can restrict further our research to separating hyperplanes containing $z$ (see Figure~\ref{subfig:HD3}). Dividing $w,b$ by the norm of $w$ will not change the zero-one loss, so we can restrict ourselves to $w\in\sphere$. Then for $u=w$ orthogonal direction defining such hyperplanes the expressions (\ref{generalTukey}) and (\ref{lossTukey}) will be equal. \qed
\end{proof}

    \begin{figure}[t!]
        \subfloat[Misclassifying $z$]{%
            \includegraphics[width=.32\linewidth,trim={30cm 1cm 15cm 20cm},clip]{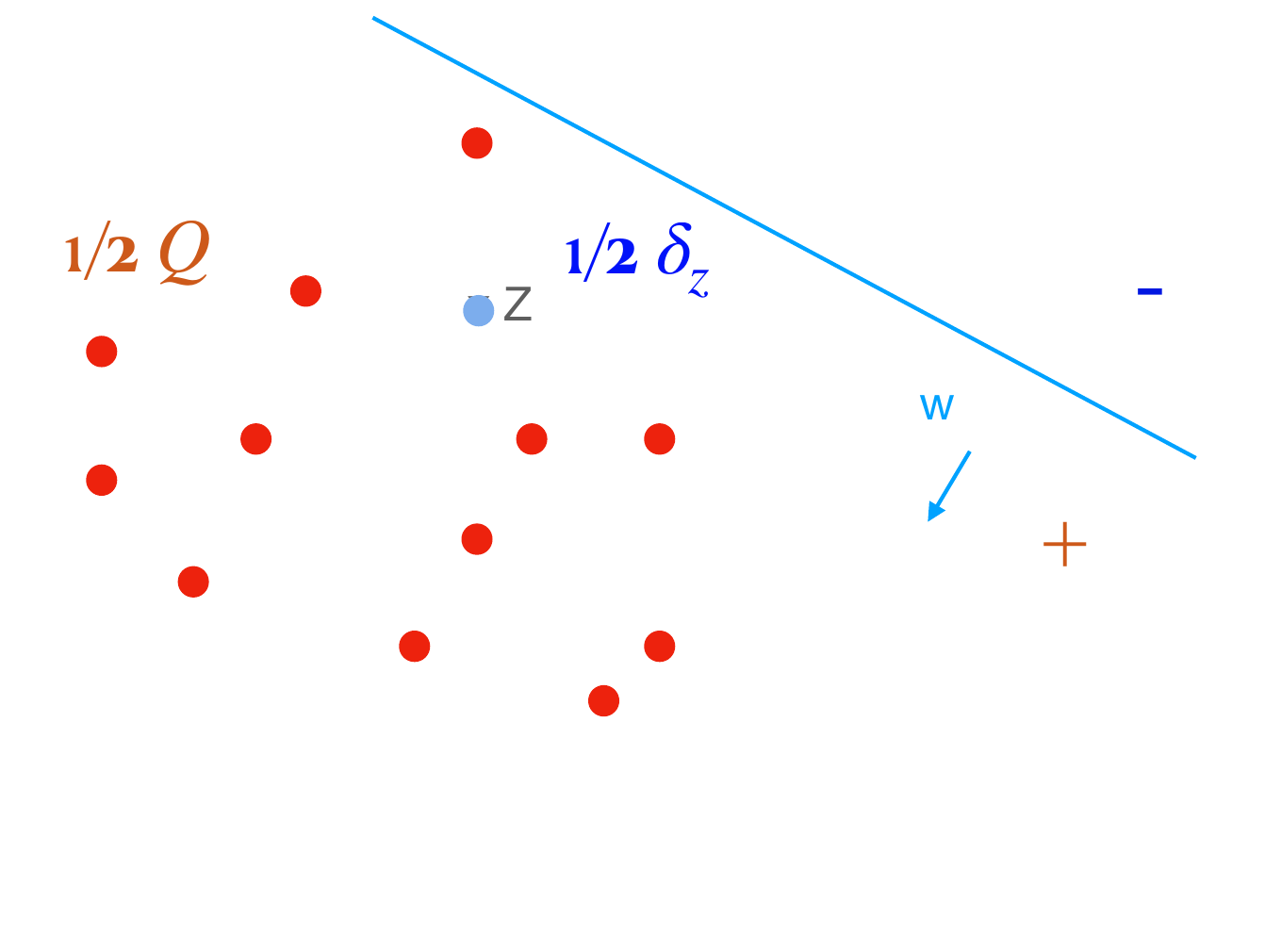}%
            \label{subfig:HD1}%
        }
        \subfloat[Misclassifying many positives]{
            \includegraphics[width=.32\linewidth,trim={30cm 1cm 15cm 20cm},clip]{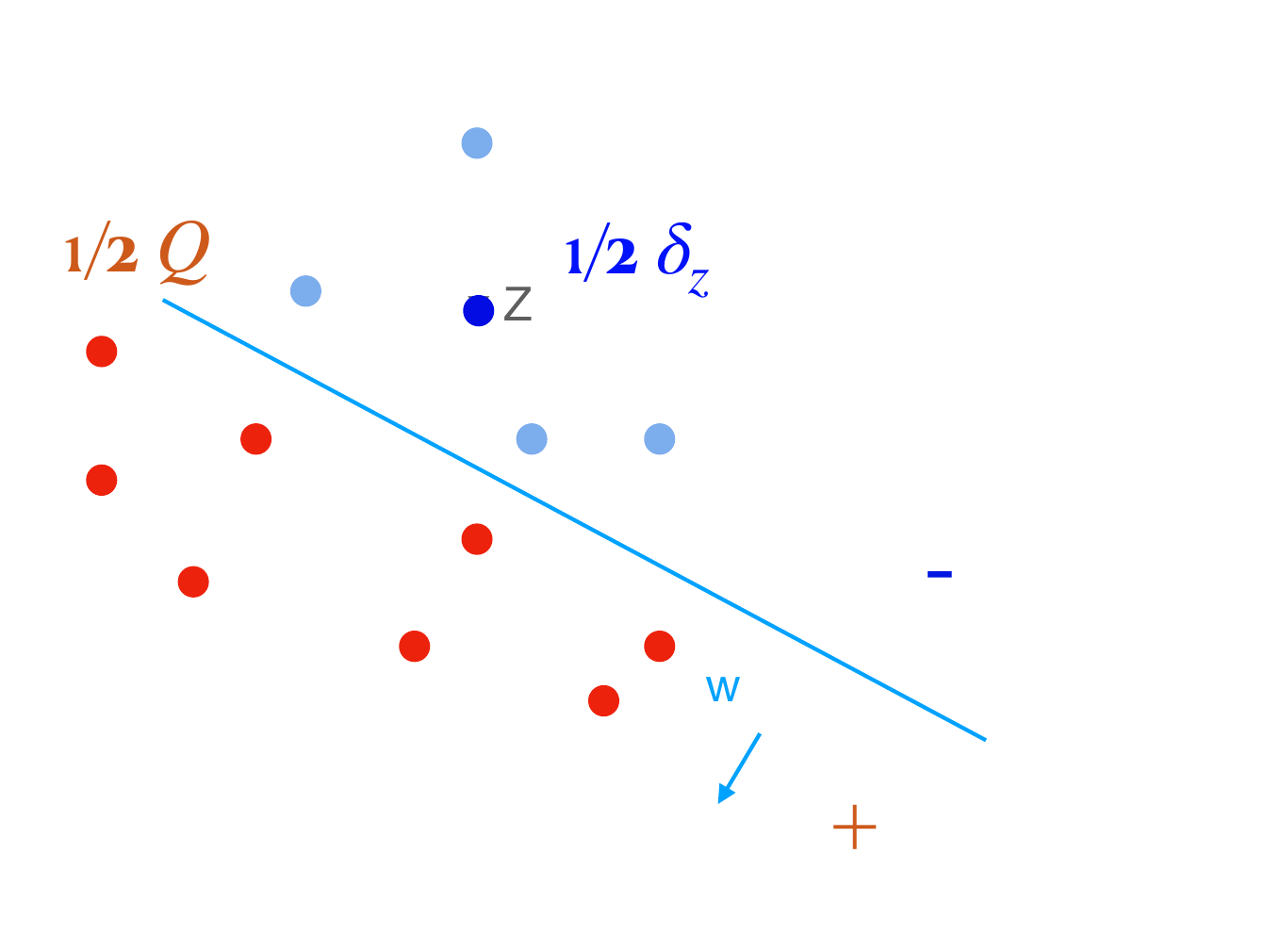}%
            \label{subfig:HD2}%
        }
        \subfloat[Recovering halfspace depth]{%
            \includegraphics[width=.32\linewidth,trim={30cm 1cm 15cm 20cm},clip]{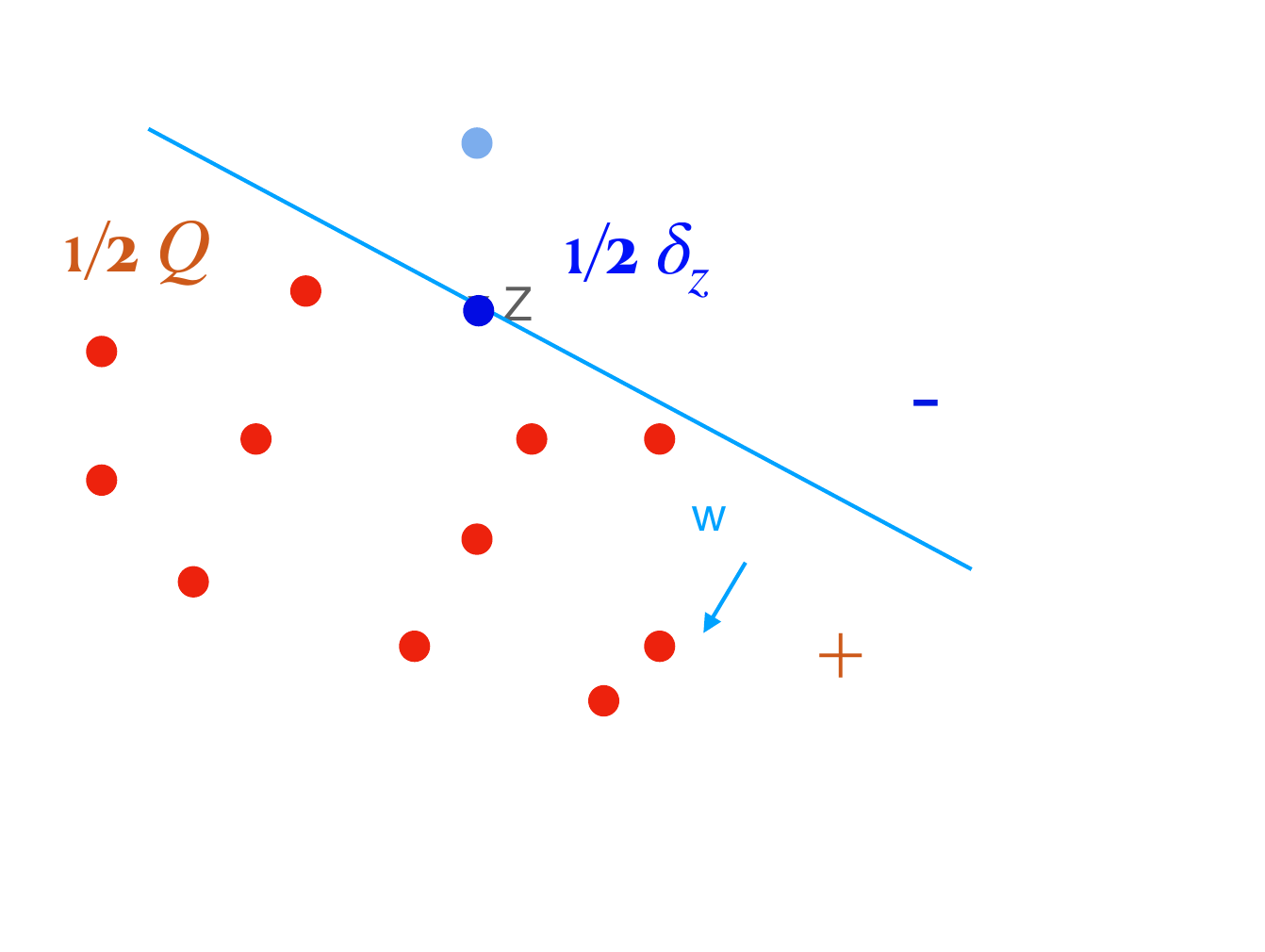}%
            \label{subfig:HD3}%
        }
        \caption{Equivalence between halfspace depth and the zero-one loss on $\LaDi$.\\
        \centering Positive points are in red, negative in blue, and if misclassified, in baby blue.}
        \label{fig:HDtoLD}
    \end{figure}

\section{Extension}\label{sec:extension}
\subsection{Depth with loss}
\begin{definition}\label{def:lossDepth}
    For some (non-empty) space of functions $\cH$ and some loss function $l$, we define the loss depth $\lossD$ on point $z$ with respect to distribution $Q$ as:
    \begin{equation*}
        \lossD(z|Q) = \inf_{f \in \cH}\bbE_{(x,y)\sim \LaDi} l(f(x),y)
    \end{equation*}
\end{definition}

Another way to express the loss depth is:
\begin{equation}
    \lossD(z|Q) = \inf_{f \in \cH} \frac{1}{2}\bbE_{x\sim Q} l_+(f(x)) +\frac{1}{2} l_-(f(z)) \label{l+l-}
\end{equation}

If $l$ and functions in $\cH$ have Lipschitz properties, this can translate to the loss depth (proof is in Appendix~\ref{appendsec:Lip}):

\begin{proposition} \label{prop:Lip}
    If $l$ is $L$-Lipschitz with respect to its first argument, all $f\in\cH$ are $W$-Lipschitz (w.r.t. metric $c$), then the depth $D_{l,\cH}$ is $\frac{1}{2}LW$-Lipschitz with respect to $z$ using metric $c$, and with respect to $Q$ using the Wasserstein distance $W_c$.
\end{proposition}

\subsection{Regularisation}

To avoid overfitting that would collapse the depth to zero for all $z\in\cX$, we want our space of classifiers $\cH$ to be just enough rich to separate some points but no more.
Therefore it is useful to apply constraints on $\cH$, however in practice it is much easier to include those constraints as a regularisation in the objective of the optimisation problem and it may lead to faster algorithms and unicity of the minimiser, therefore we also consider the regularised version of our loss depth:

\begin{definition}\label{def:regDepth}
    Denote $F$ the regularised loss $F(f,x,y) = l(f(x),y) + reg(f)$ for some regularisation function $reg: \cH \to \bbR_{\geq 0}$, such that $\bbE_{(x,y)\sim \LaDi} F(f,x,y)$ admits a unique minimum $f^*$. Then we define the regularised loss depth:
    \begin{equation*}
        \regD_{l,\cH}(z|Q) = \bbE_{(x,y)\sim \LaDi} l(f^*(x),y) .
    \end{equation*}
\end{definition}

For the rest of the paper, we will use $\regD$ instead of $D$ whenever the depth is regularised.
The penalisation function may typically include a scaling term $\lambda$ that can be tuned.


A third variant would be to use the regularised loss $F$ in Def.~\ref{def:lossDepth}; for the sake of interpretability as a pure loss, we will focus on the first two versions here.

\subsection{Illustration}\label{subsec:illu}
The halfspace depth due to its linear and convex nature has trouble to handle multimodal data. To illustrate our methods, in Fig.~\ref{fig:4ml}, we apply it to a dataset with two modes, each of 200 samples distributed as normal distribution $\cN([-3.5,-3.5],I_d)$ and $\cN([3.5,3.5],I_d)$ respectively. 

We use four classical machine learning methods from \texttt{sklearn}~\citep{scikit-learn}: 
\begin{itemize}
    \item Multi-layer Perceptron, with one hidden layer of 100 neurons with ReLu activation and log-loss and L2-regularisation with $\lambda=0.01$,
    \item Random Forest, with 1000 estimators of maximum depth of 2 and (without bootstrap),
    \item Support Vector Machine, with Gaussian kernel with $\gamma$  chosen as the inverse of the median of the squared Euclidean distance of the data and with L2-regularisation with $\lambda=1$,
    \item Gaussian process classifier, with Gaussian kernel where $\gamma$ is chosen as one half of the inverse of the squared first quartile of the Euclidean distance of the data.
\end{itemize}

    \begin{figure}[t!]
        \subfloat[Multi-Layer Perceptron Depth ($MLP\regD$)]{%
            \includegraphics[width=.48\linewidth,trim={0 0 0 0.5cm},clip]{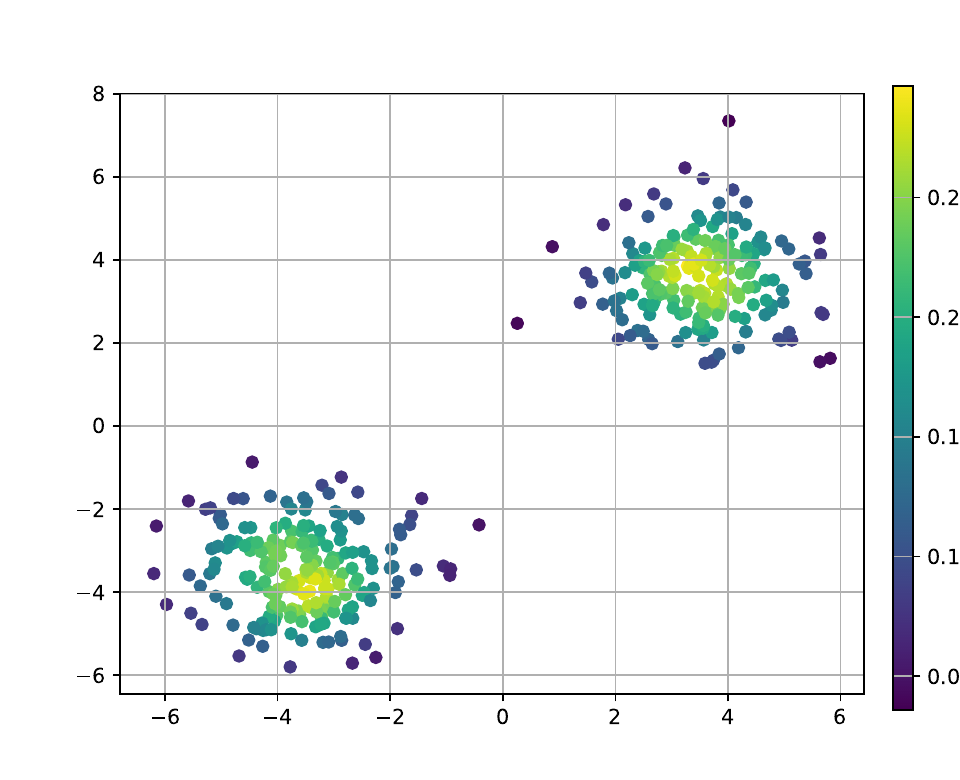}%
            \label{subfig:MLPD}%
        }\hfill
        \subfloat[Support Vector Machine Depth ($SVM\regD$)]{%
            \includegraphics[width=.48\linewidth,trim={0 0 0 0.5cm},clip]{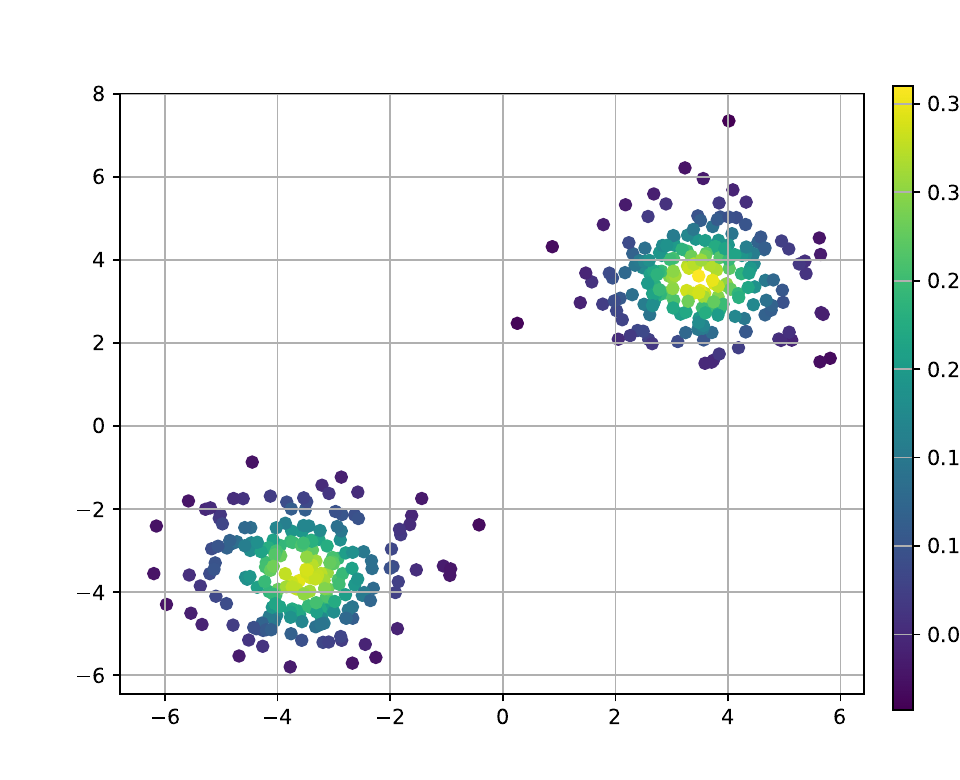}%
            \label{subfig:SVMD}%
        }\\
        \subfloat[Random Forest Depth ($RFD$)]{%
            \includegraphics[width=.48\linewidth,trim={0 0 0 0.5cm},clip]{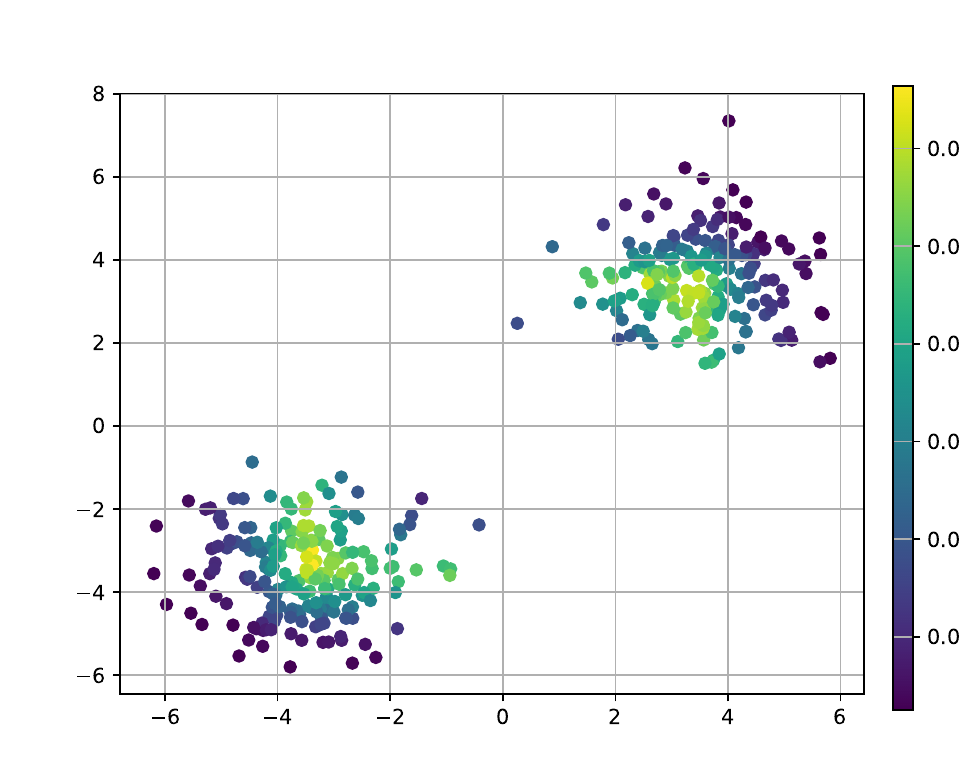}%
            \label{subfig:RFD}%
        }\hfill
        \subfloat[Gaussian Process Classifier Depth ($GPD$)]{%
            \includegraphics[width=.48\linewidth,trim={0 0 0 0.5cm},clip]{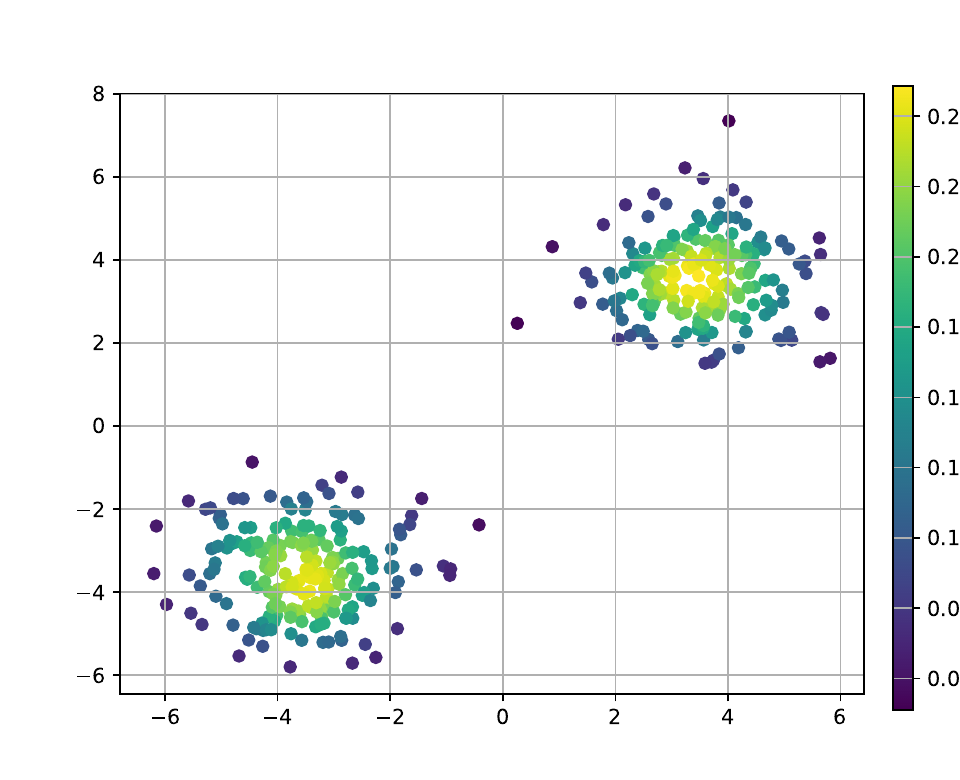}%
            \label{subfig:GPD}%
        }
        \caption{Loss depths on a bigaussian dataset with $n=400$ samples}
        \label{fig:4ml}
    \end{figure}

These algorithms recover the central part of each mode as very deep (light color) and the depth fades away the further one goes away from those centers (dark color). The random forest has more trouble with the spherical symmetry due to its dependence on the basis of the coordinates. Random rotation of the coordinates such as in~\citet{JMLR:v17:blaser16a} could help. In addition, for a similar distribution with total samples $n=200$ this time, as the dimension grows as $d=2,4,6,8$, Fig.~\ref{fig:rank} in Appendix~\ref{appendsec:ranking} displays the boxplot of the rank correlations between the ordering obtained using the true densities and the one using the depths, over 10 runs. $SVM\regD$ and $GPD$ seem to be the best performers.

\section{Statistical convergence rates}\label{sec:rates}
Here we consider some empirical distribution $Q_n$ based on $n$ samples drawn from $Q$, and want to investigate the speed of convergence of $|D(z|Q_n)-D(z|Q)|$ as $n$ grows. One of the advantage of halfspace depth is to have convergence rate 
of order $O(n^{-1/2})$, which does not depends on the dimension.
In this section, 
we develop tools to recover rates  
based on a well-known quantity: the \emph{generalisation error}, as shown in the following lemma.

\begin{lemma}\label{lemma:gen}
    The convergence rate of loss depth $\lossD$ is bounded by the generalisation error of $l_+$:
    \begin{equation*}
    | \lossD(z|Q_n)-\lossD(z|Q)| \leq \sup_{f\in\cH} |\mathbb{E}_{x\sim Q_n}l_+(f(x)-\mathbb{E}_{x\sim Q}l_+(f(x))|
    \end{equation*}
\end{lemma}
 Since the bound only depends on $Q,Q_n$, it works uniformly on all $z$.
\begin{proof}
    We will note for convenience:
    \begin{equation*}
        \Gamma(f,z,Q) = \bbE_{(x,y)\sim \LaDi} l(f(x),y) \label{eq:gamma}.
    \end{equation*}
    Then 
    \begin{equation*}
    \lossD(z|Q_n)-\lossD(z|Q) = \inf_{f\in\cH} \Gamma(f,z,Q_n) - \inf_{f\in\cH} \Gamma(f,z,Q)
\end{equation*}

For any $\varepsilon>0$, there exists some $f^*_\varepsilon\in\cH$ such that $\Gamma(f^*_\varepsilon,z,Q) \leq \inf_{f\in\cH} \Gamma(f,z,Q) + \varepsilon$. Therefore:
    \begin{align*}
    \lossD(z|Q_n)-\lossD(z|Q) &= \inf_{f\in\cH} \Gamma(f,z,Q_n) - \Gamma(f^*_\varepsilon,z,Q_n)\\ &+\Gamma(f^*_\varepsilon,z,Q_n)-\Gamma(f^*_\varepsilon,z,Q)
    +\Gamma(f^*_\varepsilon,z,Q)-\inf_{f\in\cH} \Gamma(f,z,Q)  \\
    &\leq \Gamma(f^*_\varepsilon,z,Q_n)-\Gamma(f^*_\varepsilon,z,Q) + \varepsilon \\
    &=  \mathbb{E}_{x\sim Q_n}l_+(f^*_\varepsilon(x))-\mathbb{E}_{x\sim Q}l_+(f^*_\varepsilon(x)) + \varepsilon.
\end{align*}
By symmetrical reasoning, there exists some $\hat{f}^*_\varepsilon\in\cH$ such that $\Gamma(\hat{f}^*_\varepsilon,z,Q_n) \leq \inf_{f\in\cH} \Gamma(f,z,Q_n) + \varepsilon$, and:
\begin{equation*}
    \lossD(z|Q)-\lossD(z|Q_n) \leq \mathbb{E}_{x\sim Q}l_+(\hat{f}^*_\varepsilon(x)-\mathbb{E}_{x\sim Q_n}l_+(\hat{f}^*_\varepsilon(x)) + \varepsilon
\end{equation*}
Combining those two symmetric results, taking $\varepsilon$ as small as we want allows to conclude. \qed
\end{proof}
A similar idea (see Appendix~\ref{appendsec:statrates}) works for the regularised case:
\begin{lemma}\label{lem:reggen}
    For some regularised depth $\regD$, denote $f^*_{reg}$ and $\hat{f}^*_{reg}$ respectively the minimisers in the definition of $\regD(z|Q)$ and $\regD(z|Q_n)$ as in Def.~\ref{def:regDepth}, then:
    \begin{equation*}
        | \regD(z|Q_n)-\regD(z|Q)| \leq \sup_{f\in\cH} |\mathbb{E}_{x\sim Q_n}l_+(f(x)-\mathbb{E}_{x\sim Q}l_+(f(x))| + |reg(f^*_{reg})-reg(\hat{f}^*_{reg})|
    \end{equation*}
\end{lemma}

The generalisation error has been well studied, enabling to gives convergence rates with the previous lemmas.
On the other hand, the regularisation function are often norms and we will see in the next section how to bound the last term in some case.  The minimiser depends on $z$ so the bound here is not uniform.



\section{Linear classifiers on features with $L_2$ regularisation}\label{sec:linear}
%
%
In this section, to prove theoretical results, we focus on simple classifiers of the shape 
$f:x \mapsto \langle w, \varphi(x)\rangle$ where $\varphi:\cX \to \cF$ is some defined feature map to some feature Hilbert space $\cF$, and $w\in \cF$ are linear weights. To avoid overfitting that would collapse the depth to zero, we add some (Hilbert norm) regularisation term $\lambda ||w||^2$, with $\lambda>0$. The regularised loss used in Def.~\ref{def:regDepth} is then:
\begin{equation}
    F_\lambda(w,x,y) = l(\langle w, \varphi(x)\rangle,y) + \lambda ||w||^2 \label{eq:Flinear}
\end{equation}
We also suppose $l$ convex with respect to its first argument, so that the functional $w \mapsto \bbE_{(x,y)\sim \LaDi} F_\lambda(w,x,y)$ is $2\lambda$-strictly convex with respect to $||\cdot||$, 
therefore it admits a unique minimiser, and the regularised depth is well-defined. 
As it is classically done, an intercept term $b$ can be added to consider classifiers of the shape $f:x \mapsto \langle w, \varphi(x)\rangle\ + b$, by adding an extra dimension and using the feature map $\tilde{\varphi} : x \mapsto [\varphi(x),1]$ and weight $\tilde{w} = [w,b]$.
Furthermore, we may use the following assumptions:
\begin{assumptionp}{W} \label{assum:W}
    All weights $w\in\cF$ are such that $||w||\leq W$.
    \end{assumptionp}
\begin{assumptionp}{B} \label{assum:B}
    All $x$ in the support of $Q$ are such that $||\varphi(x)||\leq B$.
    \end{assumptionp}

We further focus on classical well-known machine learning algorithms with convex loss:
\begin{description}
    \item \emph{Logistic Regression Depth} (${LR\regD}$): depth based on the logistic regression with:
    \begin{itemize}
         \item \textbf{Features}: plain features $\varphi: \bbR^d \to \bbR^{d+1}, x \mapsto [x,1]$.
         \item \textbf{Loss}: Logistic loss defined by $l_{log}(f(x),y) = (1+e^{-yf(x)})$.
    \end{itemize}
    \item \emph{Support Vector Machine Depth} (${SVM\regD}$): depth based on (soft-margin) Support Vector Machine with kernel methods, in particular:
    \begin{itemize}
        \item \textbf{Features}: $\varphi$ being the canonical features map of some bounded kernel in some Reproducible Kernel Hilbert Space (see~\citet{berlinet2011reproducing} and some background in Appendix~\ref{appendsec:RKHS}).
        \item \textbf{Loss}: Hinge loss defined by $l_{hinge}(f(x),y)= \max(0,1-yf(x))$.
    \end{itemize}
\end{description}
For logistic regression, we also consider the unregularised version $LRD$.
Notice that for the unregularised $LRD$ with Assumption~\ref{assum:W}, the optimisation problem:
\begin{equation*}
\begin{aligned}
\min_{w} \quad & \bbE_{(x,y)\sim\LaDi}l_{log}(\langle w,x\rangle)\\
\textrm{s.t.} \quad & ||w||^2 \leq W^2\\
\end{aligned}
\end{equation*}
has a convex objective with convex constraint on a convex domain and is feasible, then it verifies Slater's condition (see for instance~\citet{Boyd_Vandenberghe_2004}) so it has strong duality with its Lagrangian dual problem, which corresponds to the regularised $LR\regD$ for some $\lambda$ that depends however on $W$ and $z$, and therefore may be different for different points.

On the other hand, for regularised versions, Assumption~\ref{assum:W} is implied:
\begin{proposition}
    For any regularised depth with regularised loss $F_\lambda$ with $L_2$ regularisation as in ~\eqref{eq:Flinear}, Assumption~\ref{assum:W} can always be considered verified.
\end{proposition}
\begin{proof}
    When plugging the candidate $w=0$ into the regularised objective, we get some value $D = \bbE_{(x,y)\sim \LaDi} F_\lambda(0,x,y) = \frac{1}{2}(l_+(0)+l_-(0))$. 
    We deduce that since $\bbE_{(x,y)\sim \LaDi} F_\lambda(w_\lambda^*,x,y) \leq \bbE_{(x,y)\sim \LaDi} F_\lambda(0,x,y)$, we must have $||w_\lambda^*||\leq \sqrt{\frac{D}{\lambda}}$. So we can restrict our search during minimisation to $w$ such that $||w||\leq\sqrt{\frac{D}{\lambda}}$, and consider Assumption~\ref{assum:W} verified for $W=\sqrt{\frac{D}{\lambda}}$. This may be helpful to take into consideration when choosing $\lambda$. \qed
\end{proof}

Assumption~\ref{assum:B} is verified for $LRD$/$LR\regD$ when the distribution $Q$ has bounded support, and also verified for $SVM\regD$ when using a bounded kernel ($\forall x \in \cX, ||\varphi(x)||^2=k(x,x)\leq B$). This covers a lot of kernels such as Gaussian, Laplacian and inverse multiquadric among others.

\subsection{Properties}

The following properties are properties that are often satisfied by classical depths.

$\triangleright$  \textbf{Boundedness}
\begin{proposition}
    Loss depths (regularised or not) with linear weights as in eq.~(\ref{eq:Flinear}) using the logistic loss or the hinge loss give values in $[0,1]$.
\end{proposition}
\begin{proof}
    For the logistic and hinge losses, using $w=0$, we get $\bbE_{(x,y)\sim \LaDi} F_\lambda(0,x,y) = \frac{1}{2}(l_+(0)+ l_-(0)) = 1$, so the depth will be smaller or equal to $1$. \qed
\end{proof}

$\triangleright$  \textbf{Convexity with plain features}
\begin{proposition}\label{prop:convex}
    For loss depth as in Def.~\ref{def:lossDepth} with a monotone loss $l_-$ and using plain features ($\varphi$ being identity or $\varphi:x \mapsto [x,1]$ for the intercept term), the depth is convex i.e. the depth regions $D_\alpha$ are convex sets.
\end{proposition}

The proof is in Appendix~\ref{appendsec:convex}, and also works for the third variant of loss depth using 2-norm regularisation inside Def.~\ref{def:lossDepth}, where the regularisation is included into the total loss function.

So by monotonicity of the logistic loss:
\begin{corollary}
    $LRD$ is a convex depth.
\end{corollary}

$\triangleright$  \textbf{Lipschitzness}
We have the following corollary:
\begin{corollary}
    $LRD$ with Assumption~\ref{assum:W}, $LR\regD$ and $SVM\regD$ with Gaussian kernel are Lipschitz with respect to $z$ using the Euclidean distance and with respect to $Q$ using the $W_c$ distance with $c$ as Euclidean distance. 
\end{corollary}
\begin{proof}
    We denotes $d_\varphi(x,y) = ||\varphi(x)-\varphi(y)||$.
    When Assumption~\ref{assum:W} is verified, if $d_\varphi(x,y)\leq K||x-y||$ for some constant $K$, then $f:x\mapsto \langle w, \varphi(x)\rangle$ is $WK$-Lipschitz by Cauchy-Schwartz. In case of $LRD$, $d_\varphi$ is just the usual Euclidean distance of $\bbR^d$ so it is true with $K=1$. For $SVMD$ with Gaussian kernel, we have $d_\varphi(x,y) = \sqrt{2(1-k(x,y))} \leq  \sqrt{2\gamma}||x-y||$ using $1-e^{-x}\leq x$, therefore it is true with $K=\sqrt{2\gamma}$.
    Proposition~\ref{prop:Lip} allows to conclude since the hinge loss and logistic loss are $1$-Lipschitz. \qed
\end{proof}

$\triangleright$  \textbf{Upper semicontinuity}
Upper semicontinuity means that the depth regions $D_\alpha = \{z|\lossD(z|Q)\geq \alpha \}$ are closed.
Since $SVM\regD$ and $LR\regD$ (or $LRD$ with Assumption~\ref{assum:W}) are Lipschitz with respect to $z$, they are continuous and therefore their depth regions are closed.

$\triangleright$  \textbf{Computational complexity with guarantees}
For $SVM\regD$, we can solve the problem in $O(n^2)$ and under some conditions get an $\varepsilon$-solution in $O(n^{1+o(1)}\log(1/\varepsilon))$\citep{Gu25}, and for $LR\regD$ by smoothness and strong convexity (see Appendix~\ref{appendsec:optim}), it is solvable by gradient descent in $O(n d \log(1/\varepsilon))$ for dimension $d$.

$\triangleright$  \textbf{Statistical convergence rates}
Using the lemmas from section~\ref{sec:rates}, we prove convergence rates for classical methods by bounding their generalisation error. To do so, it will be useful to consider the Rademacher complexity of our set of classifiers $\cH$ for $n$ samples, defined as:
\begin{equation*}
    \cR_{Q,n}(\cH) = \mathbb{E}_{(x_1,\dots,x_n) \sim Q^n} \mathbb{E}_{(\sigma_1,...,\sigma_n) \sim Rad^n} \sup_{\mathcal{F}} (\frac{1}{n}\sum_{i=1}^n \sigma_i f(z_i))
\end{equation*}
where the Rademacher distribution $Rad$ is uniform over $\{-1,1\}$. Rademacher complexities are used to bound generalisation errors. For linear classifiers verifying Assumptions~\ref{assum:B} and \ref{assum:W}, \citet{kakade2008complexity} showed Rademacher complexity of order $O(n^{-1/2})$. Combining with Lemma~\ref{lemma:gen} and \ref{lem:reggen} (see full proof in Appendix~\ref{appendsec:statrates}) gives:
\begin{theorem}\label{prop:rates}
    Under Assumptions~\ref{assum:B} and \ref{assum:W}, $LRD$ and $SV\regD$ have convergence rates:
    \begin{equation*}
        \sup_{z\in\mathcal{X}}|LRD(z|Q_n)-LRD(z|Q)| = O_Q(n^{-1/2}), \text{ and }
    \end{equation*}
    \begin{equation*}
            |SV\regD(z|Q_n)-SV\regD(z|Q)| = O_Q(n^{-1/2}).
        \end{equation*}
\end{theorem}

\section{Experiments}\label{sec:XP}
\subsection{Contamination}\label{sec:conta}
\paragraph{OCSVM vs $SVM\regD$ on contamination}
Data depths aim at being robust while totally unsupervised, which contrasts with many algorithms such as One-Class SVM (OC-SVM \citep{ocsvm}), deep Support Vector Data Descriptor (SVDD~\citep{ruff18a}), OC-GAN ~\citep{ocgan} among others, who actually uses a training dataset of supposedly ``inliers". However nothing guarantees that this training dataset is free from contamination at the beginning. In particular for OC-SVM, one has to set $\nu$ higher than the expected fraction of outliers but it may not be known in advance.  Here we consider a dataset of $n=200$ samples coming from a Gaussian $\cN([-1,-1],I_d)$ but where $10\%$ of the points have been replaced using $\cN([2,2],I_d)$. Fig.~\ref{fig:conta} displays a heatmap (the redder the deeper) of the scores with 5 score-level zones based on the top 5 quantiles for OC-SVM with $\nu =0.15$ as well as $SVM\regD$ (with $\lambda=1$), both with Gaussian kernel and $\gamma=1$ to match the variances of the Gaussian modes. Even with $\nu$ above the true contamination rate, OC-SVM gives a high score to the contamination cluster and good scores to the `bridge' between the two clusters. On the other hand with $SVM\regD$, the highest decile zone contains most of the authentic cluster and leaves out most of the contamination points.

    \begin{figure}[t!]
        \subfloat[$OC\text{-}SVM$]{
            \includegraphics[width=0.48\linewidth, trim={2cm 0.5 2cm 1.2cm},clip]{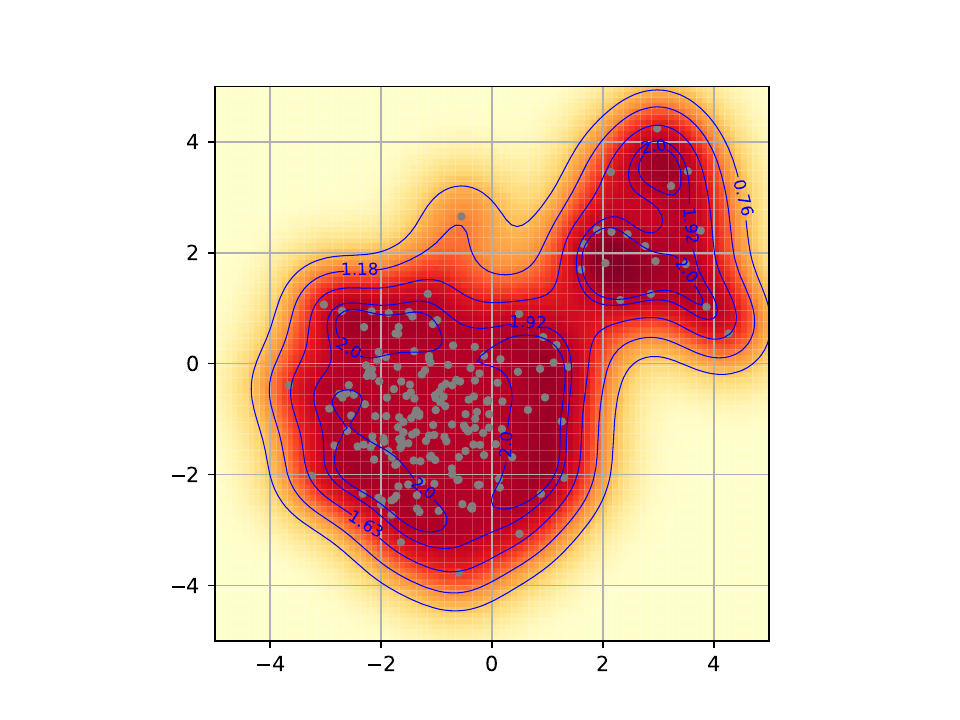}%
        }
        \subfloat[$SVM\regD$]{%
            \includegraphics[width=0.48\linewidth,, trim={2cm 0.5 2cm 1.2cm},clip]{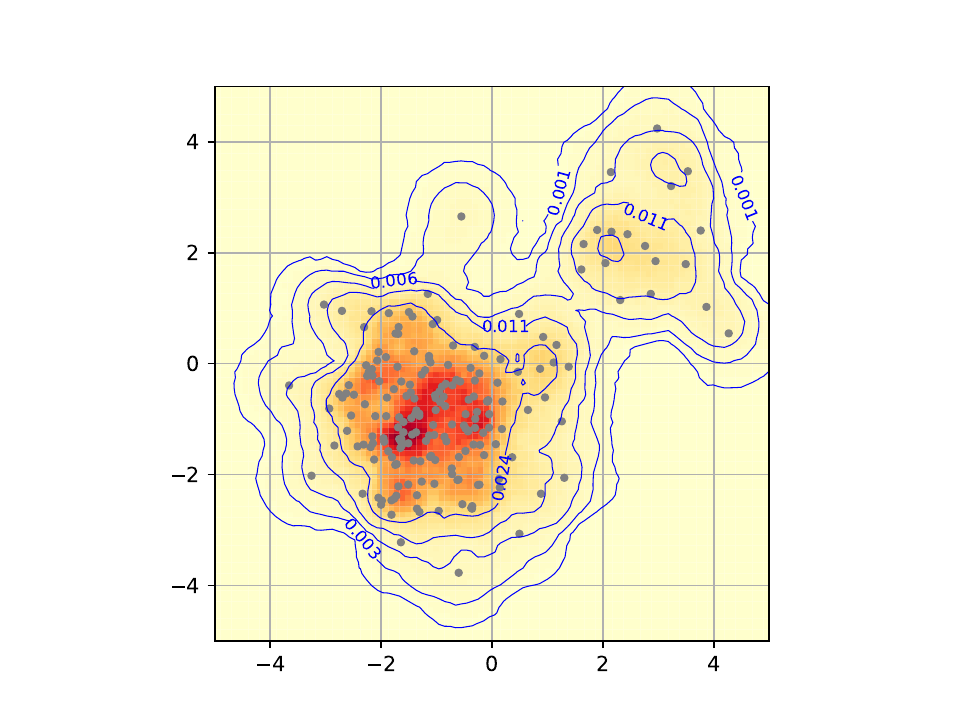}%
        }
        \caption{Score heatmaps with top 5 quantile zones of  $n=200$ samples from $\cN([-1,-1],I_d)$, with a $10\%$ contamination by the upper-right cluster centered in $[2,2]$. New points in the contaminated zone will be classified as inlier by OC-SVM.}
    \label{fig:conta}
    \end{figure}

\subsection{Anomaly Detection}\label{sec:AD}
\subsubsection{Computer Vision}\label{sec:CV}
\paragraph{Interpretability of Logistic Regression Depth}
Here we carried anomaly detection using a dataset of $n=1000$ samples of class 5 (\texttt{`sandal'}) from Fashion-MNIST~\citep{xiao2017/online} as distribution $Q_n$, and compute $LR\regD(\cdot|Q_n)$ (with $\lambda=1$) on another 1000 samples of class 5 (inliers) and 1000 samples from different classes (outliers). 
In Fig.~\ref{fig:interpret}, we display the deepest and shallowest outliers and inliers: the typical sandal of the dataset is a flip-flop so high-heel sandals stand out, 
and the closest outliers are sneakers rather than T-shirts as one can expect.
We also display the regression coefficients used to compute their $LR\regD$: whiter pixels corresponds to negative coefficients, greyish to zero, and blacker to positive coefficients. The negative coefficient is what allows the classifier to separate the sample from the rest of the data, while the black coefficients corresponds to pixels which are necessary for the class (of course reversing the colors of the dataset would reverse the reasoning). We can therefore observe that the sneaker is betrayed by its air cushion, while on the T-shirt, we can see the shadow of high-heel sandals. Thus we can achieve easy interpretability using a simple linear model when deep convolutional architectures are not necessary. In Appendix~\ref{appendsec:mnist}, we also display AUC scores on MNIST~\citep{mnist} and Fashion-MNIST datasets for $LR\regD$ and $SVM\regD$, and show that using a dataset of $n=1000$ or even $100$ samples achieve performances not far from using the full dataset, speeding up the process. In particular, $LR\regD$ (one neuron, the shallowest architecture!) achieve competitive scores close to some deep methods such as deep SVDD.
    \begin{figure}[!ht]
        \subfloat[Deepest inlier]{%
            \includegraphics[width=.24\linewidth, trim={2cm 1.5 2cm 1.5cm},clip]{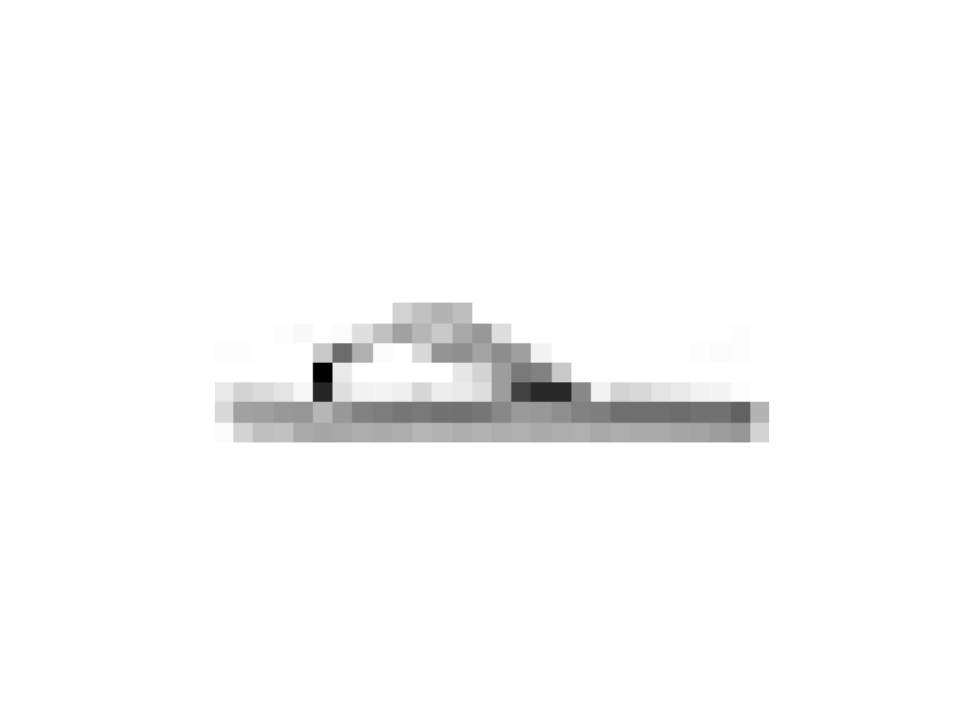}%
            \includegraphics[width=.24\linewidth, trim={2cm 1.5 2cm 1.5cm},clip]{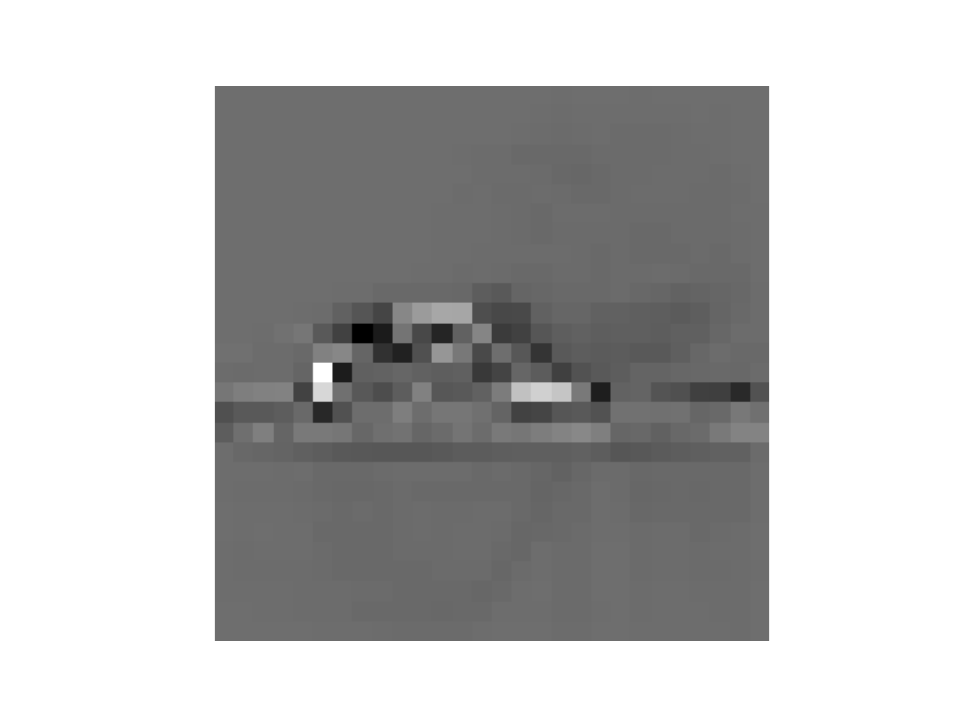}%
        }\hfill
        \subfloat[Least deep inlier]{%
            \includegraphics[width=.24\linewidth, trim={2cm 1.5 2cm 1.5cm},clip]{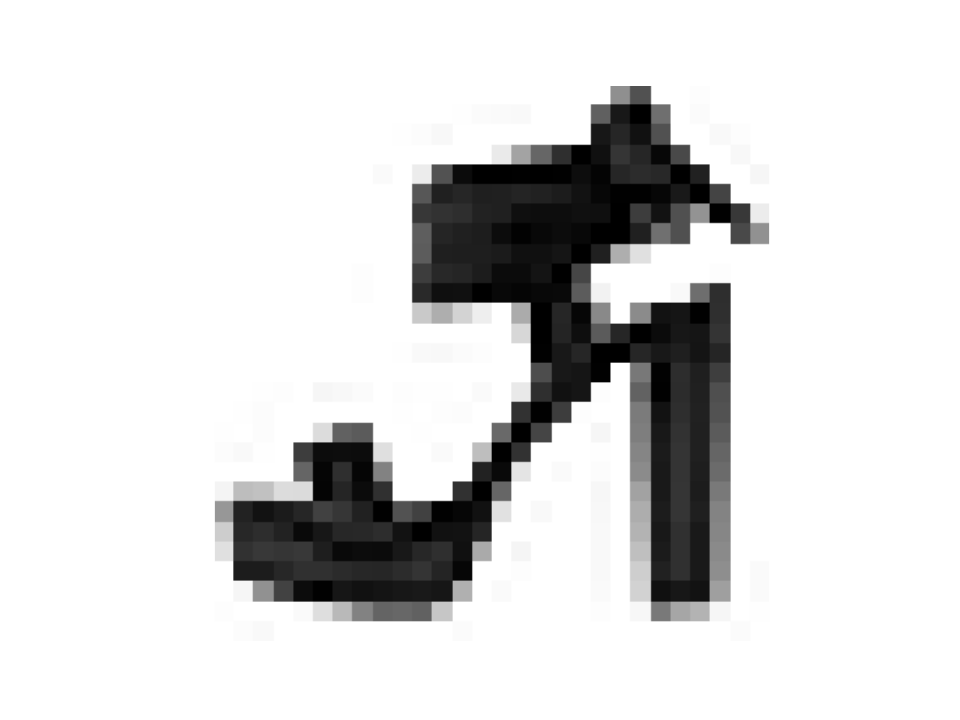}%
            \includegraphics[width=.24\linewidth, trim={2cm 1.5 2cm 1.5cm},clip]{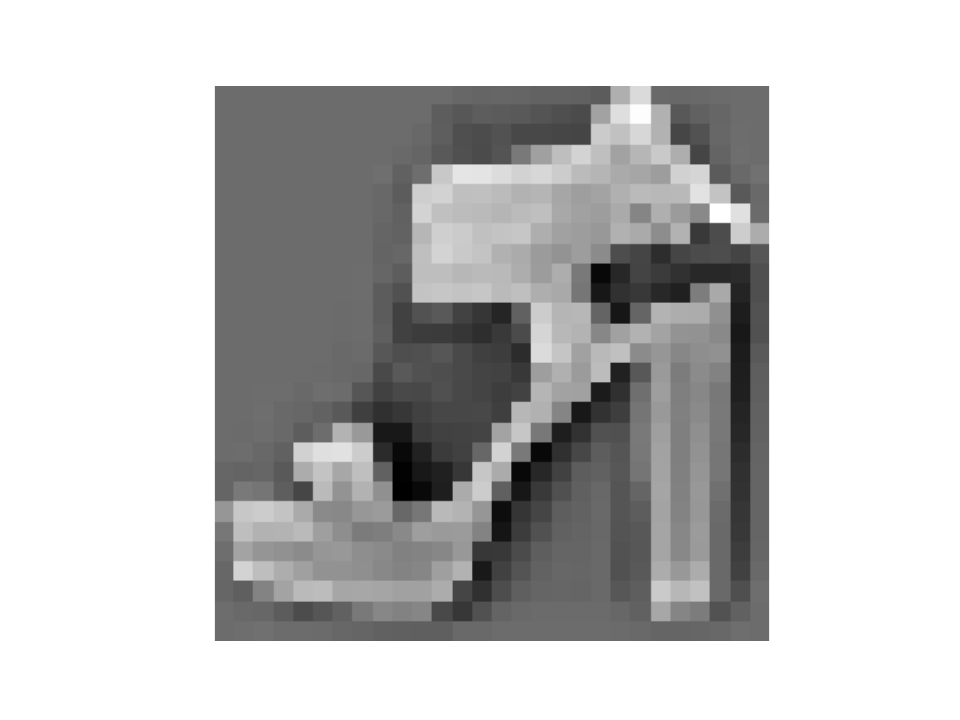}%
        }\\
        \subfloat[Deepest outlier]{%
            \includegraphics[width=.24\linewidth, trim={2cm 1.5 2cm 1.5cm},clip]{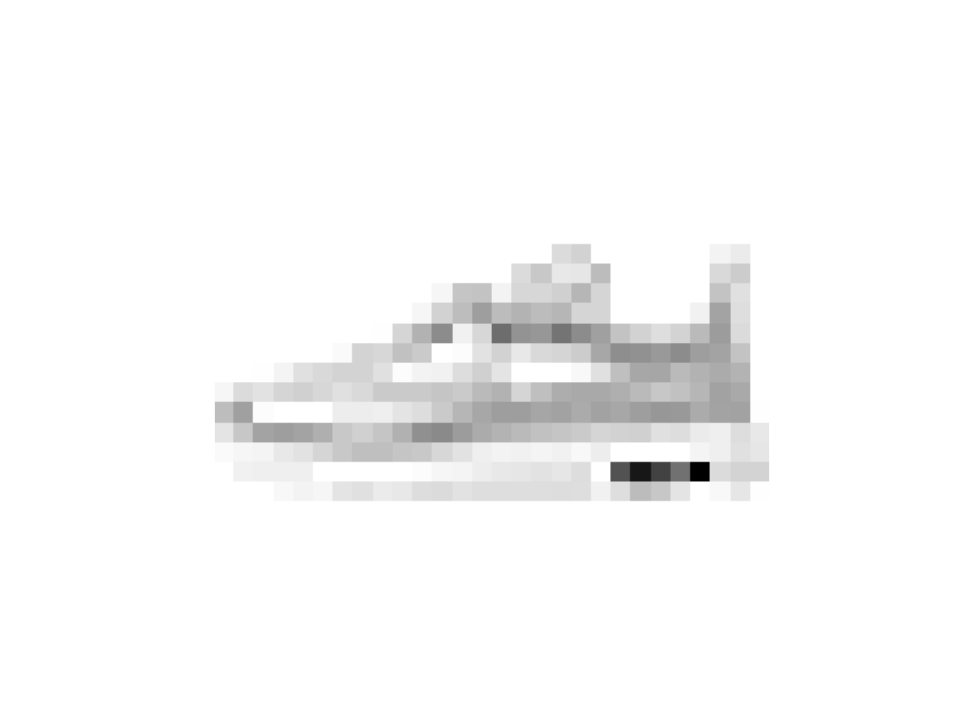}%
             \includegraphics[width=.24\linewidth, trim={2cm 0.5 2cm 1.2cm},clip]{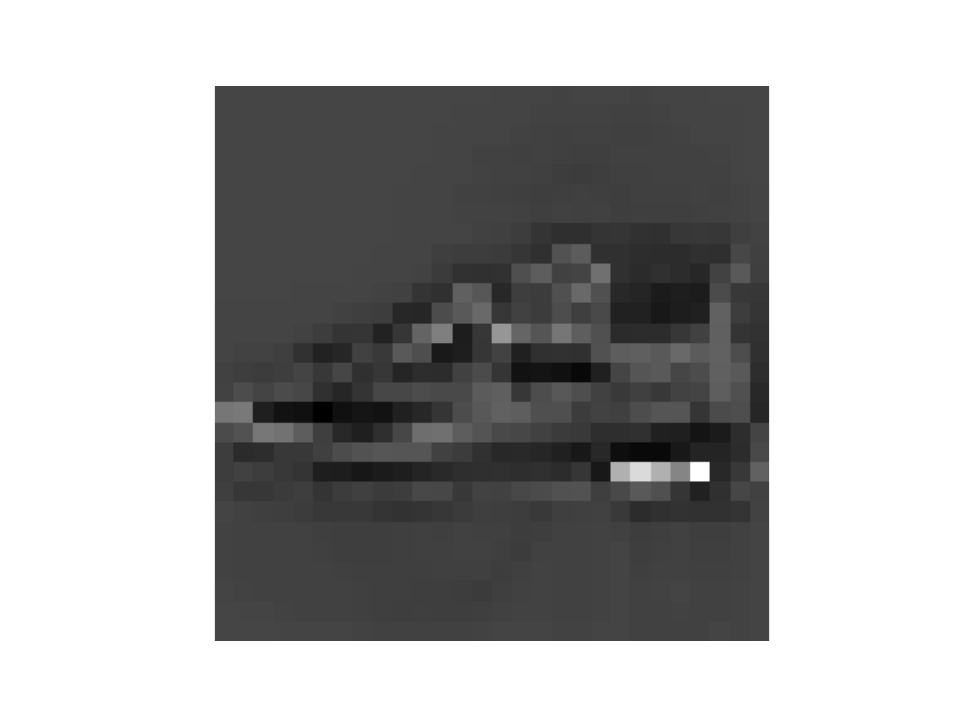}%
        }\hfill
        \subfloat[Least deep outlier]{%
            \includegraphics[width=.24\linewidth, trim={2cm 1.5 2cm 1.5cm},clip]{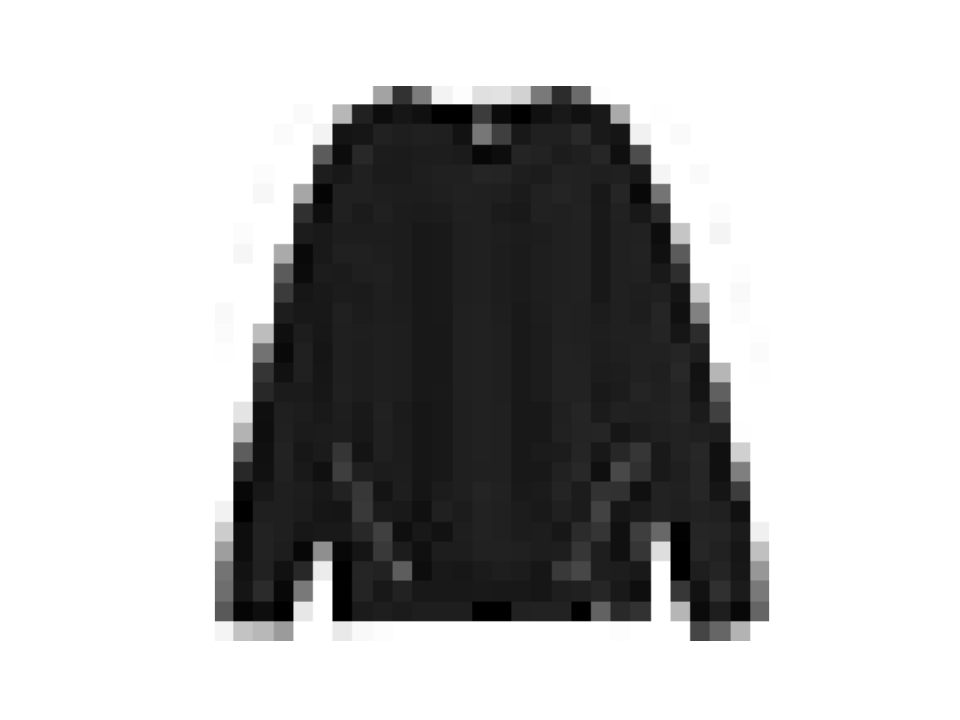}%
            \includegraphics[width=.24\linewidth, trim={2cm 1.5 2cm 1.5cm},clip]{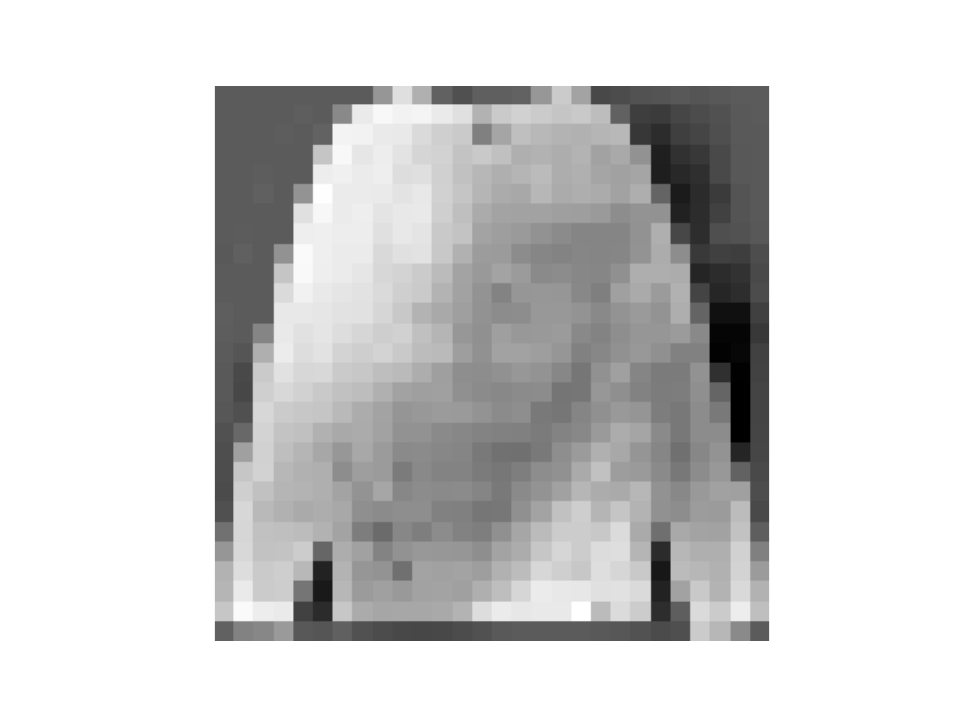}%
        }
        \caption{Outstanding samples for the \texttt{`sandal'} class of Fashion-MNIST on the left side and on the right side the regression coefficients used to compute their Logistic Regression Depth $LR\regD$}
        \label{fig:interpret}
    \end{figure}
\vspace{-0.4cm}
\subsubsection{Outlier Dectection Data Sets (ODDS)}
We carried out anomaly detection on some ODDS datasets~\citep{odds} with the same depths' settings as in subsection~\ref{subsec:illu} (and as in previous section for $LR\regD$). Since Gaussian process classifier is slow for huge datasets (naive complexity is cubic in $n$), we did not run it for the \texttt{thyroid}, \texttt{annthyroid} and \texttt{pendigits} datasets, however it performs well on some smaller datasets. Moreover, we compare to competitors OC-SVM (with same kernel as $SVM\regD$) and Local Outlier Factor (LOF)~\citep{breunig00} with number of neighbours $k \in [5, 10, 15, 20, 30]$.
\begin{table}[ht]
\centering
\caption{AUC-ROC on anomaly detection datasets\\
Best scores in bold, second best in italics.}\label{tab:auroc} 
\begin{tabular}{lrrrrrrr}
\toprule
 & $OCSVM$ &  $LOF$ & $LR\regD$ & $SVM\regD$ & $RFD$ & $MLP\regD$ & $GPD$\\
 \midrule
wine & 0.85  & \textbf{1.00} & 0.89 & 0.97 & 0.63 & 0.33 & \textit{0.99}\\
glass & 0.57 & \textbf{0.87} & 0.55 & 0.81 & 0.56 & 0.72 & \textit{0.86}\\
vowels & 0.69  & \textbf{0.95} & 0.68 & 0.78 & 0.66 & 0.88 & \textit{0.90} \\
pima & 0.59  & 0.59 & \textit{0.61} & \textbf{0.67} & 0.58 & 0.56 & \textit{0.61} \\
breastw & 0.41  & 0.46 & \textit{0.97} & \textbf{0.99} & \textit{0.97} & 0.95 & \textit{0.97} \\
lympho & 0.92  & \textit{0.98} & \textit{0.98} & 0.87 & \textbf{0.99} & 0.97 & \textbf{0.99}\\
thyroid & 0.91 & 0.89 & 0.83 & \textbf{0.99} & \textit{0.97} & 0.70 & N/A\\
annthyroid & 0.66 & 0.74 & 0.54 & \textbf{0.89} & \textit{0.77} & 0.50 & N/A \\
pendigits & 0.86 & 0.54 & \textbf{0.93} & 0.83 & \textit{0.91} & \textbf{0.93} & N/A\\
\bottomrule
\end{tabular}
\end{table}

\section{Conclusion}
\paragraph{Discussion}
We introduce a general framework to convert machine learning algorithms into depth functions, inheriting optimisation and statistical properties.
Loss depths beg the fundamental question: what is the complexity of the task? 
Often the focus is on getting the best scores 
inciting to use the best architecture, using deeper and deeper networks. Recent works have tried to compress the size of the neural networks, with applications to edge device~\citep{compression23}. Here we show that actually using classifiers which are precisely borderline to classify the dataset allows to characterise well the distribution of such dataset. This may be connected to the idea of growth function~\citep{vapnik68}. The distribution of the data depth values (not too high, not too uniformly low) could be a way to decide which size of classifiers describes appropriately some dataset.
\paragraph{Limitations}
In order to achieve robustness, data depths run some optimisation for each point, so on huge dataset, one-class novelty detection methods learning on a training set should be faster at test time. However, 
the data depths of each point can also be computed in parallel thanks to their unsupervised independence. Since the distribution of positively-labelled data is the same in all depth computations, pre-training could also be considered in future works.
Loss depths can be relevant for some real-life scenarios having only one new datum 
per day, such as weather or financial forecasts. Another caveat is that since the negative examples consist of exactly one point, the sets of classifier should take into account its mass, 
unlike decision trees. It should also be discriminative (model $p(\mathrm{label}|\mathrm{data})$) rather than generative (model $p(\mathrm{data}|\mathrm{label})$).
\paragraph{Outlook}
The loss of the logistic regression can be seen as the cross-entropy applied to the Bayesian posterior, and it would be interesting to interpret the depth from an information or Bayesian viewpoint.
One could get a bound on the regularised $LR\regD$ using some results of~\citet{marteau-ferey19a} on the convergence of the regularised minimiser under some condition on $\lambda$. Is it extendable without a bound on $\lambda$?
Analysis of other classifiers' convergence rates, more robust classifiers, as well as the use of neural networks with higher number of layers on harder datasets is left open for future works. 
\bibliographystyle{plainnat}
\bibliography{main}

\newpage
\appendix

\section{Proofs}\label{appendsec:proofs}

\subsection{Background on Reproducing Kernel Hilbert Space}\label{appendsec:RKHS}

    Kernel methods allow to tackle nonlinear relationship between data by applying linear tools to non-linear features of the data. They rely on the key notion of Reproducing Kernel Hilbert Space \citep{aronszajn1950theory, berlinet2011reproducing} and the famous kernel trick. For a positive definite symmetric kernel $k: \cX \times \cX \to \bbR$ defined over some 
    non-empty set $\cX$, Aronszajn's theorem says that there is a unique Hilbert space $\cH$ equipped with an inner product $\langle \cdot, \cdot \rangle_{\cH}$, for which $k(\cdot, x) \in \cH$ and $k$ is its reproducing kernel, that is:
   \begin{equation*}
       \forall f \in \mathcal{H},\forall x \in X, \langle f, k(\cdot,x)\rangle_\mathcal{H} = f(x).
   \end{equation*}
   It admits a canonical feature map $\varphi: \cX \to \cH$ defined by: $\varphi(x):=k(\cdot,x) \in \cH$.
   This defines what is called a kernel inner product on the space $X$:
       \begin{equation*}
           k(x,y) = \langle \varphi(x), \varphi(y) \rangle_\mathcal{H}
       \end{equation*}
    In practice, an explicit formula often allows to compute the kernel inner product without going through the (potentially) infinite dimensional RKHS - this is called the {\it kernel trick}.
    The dual of the Euclidian space $\bbR^d$ consisting of 
    $\big(\bbR^d\big)^*=
    \{w^*: \bbR^d \to \bbR, x \mapsto \langle w,x\rangle | w\in \bbR^d\}$ is a trivial RKHS for $\cX=\bbR^d$ defined such that $\varphi : w \mapsto w^*$ and $\langle w^* , v^*\rangle_\cH = \langle w^* , \varphi(v) \rangle_\cH =  w^*(v)=\langle w , v\rangle$.
    Some well-known bounded kernels among others are:
    \begin{itemize}
        \item the Gaussian kernel: $k(x,y)=e^{-\gamma||x-y||^2}$, $\gamma>0$
        \item the Laplacian kernel: $k(x,y)=e^{-\frac{||x-y||_1}{\sigma}}$, $\sigma>0$
        \item the inverse multi-quadric (IMQ) kernel: $k(x, y) = (c^2+||x-y||^2)^\beta$ for some $\beta < 0$ and $c > 0$.
    \end{itemize}


\subsection{Optimisation}\label{appendsec:optim}

This background's material can be found in textbooks such as, e.g., Chapter 5 of \citet{bach2024learning} or in~\citet{garrigos2023handbook}:

\begin{definition}
    A function $F: \bbR^d \to \bbR \cup \{\infty\}$ is said to be $\mu$-strongly convex, (with $\mu>0$), if and only if:
        \begin{equation*}
            \forall x,y \in \bbR^d, \forall t \in [0,1] \in tF(x) + (1-t) F(y) \geq F(tx+(1-t)y) + \mu \frac{t(1-t)}{2} ||x-y||^2
        \end{equation*}
\end{definition}

In particular, by Lemma 2.12 of~\citet{garrigos2023handbook}, it is equivalent to $F$ being decomposable as $F(x) = l(x) + \frac{\mu}{2} ||x||^2$ for all $x$ where $l$ is a convex function.
Therefore the regularised loss $F_\lambda$ of $LR\regD$ is $2\lambda$-strongly convex by convexity of $l_{log}$.

\begin{definition}
     A function $F: \bbR^d \to \bbR$ is $L$-smooth (for $L>0$) if it is differentiable and its gradient is $L$-Lipschitz.
\end{definition}

If $F$ is twice differentiable, a sufficient condition is to bound the norm of the Hessian $||\nabla^2 F ||$ by $L$.
The interesting case is when a function is both $\mu$-strongly convex and $L$-smooth (then according to Remark 2.27 of \citet{garrigos2023handbook} we necessarily have $\mu\leq L$).
In such case, Theorem 3.6 and its Corollary 3.8 of \citet{garrigos2023handbook} guarantee that the Gradient descent algorithm applied in $t$ steps with stepsize $0<\gamma\leq\frac{1}{L}$ starting from some point $x^{(0)}$ can produce some $x^{(t)}$ as close to the minimiser $x^*$ as:
\begin{equation}
    ||x^{(t)}-x^*||^2 \leq  (1-\gamma\mu)^t||x^{(0)}-x^*||^2
\end{equation}

In particular if $\gamma = \frac{1}{L} \leq \frac{1}{\mu}$, one can get an $\varepsilon$-approximated solution in $O(\log(\frac{1}{\varepsilon}))$ iterations.
This can be use to show optimisation convergence guarantee for $LR\regD$.

\begin{proposition}
    Suppose distribution $Q$ has bounded support (Assumption~\ref{assum:B}).
    When $F_\lambda$ is the regularised logistic loss, then the empirical objective $\Gamma_\lambda : w \mapsto \bbE_{(x,y)\sim \LaDis{Q_n}{z}} l_{log}(\langle w, x \rangle, y) + \lambda ||w||^2$ is $2\lambda$-strongly convex and $L$-smooth, with $L\geq 2\lambda$. 
    Therefore one can find in time $O(nd\log(\frac{1}{\varepsilon}))$ some $O(\varepsilon)$-approximation 
    of the value of $LRD(z|Q_n)$.
\end{proposition}

\begin{proof}
    The $2\lambda$-strongly convexity has already been proven. Notice that the individual loss can be expressed $l_{xy} : w \mapsto l_{log}(\langle w, x \rangle, y) = -\log(s(\langle w,yx\rangle))$ where $s$ is the sigmoid function $s: z \mapsto \frac{1}{1+ e^{-z}}$. Using $s'(x) = s(x)(1-s(x))$ and $(1-s(x))=s(-x)$ deriving gives : 
    \begin{equation*}
        \nabla l_{xy}(w) =  -yx s(-\langle w,yx\rangle)
    \end{equation*}, then deriving a second time gives the Hessian: 
    \begin{equation*}
        \nabla^2 l_{xy}(w) = xx^T s(-\langle w,yx\rangle)(1-s(-\langle w,yx\rangle))
    \end{equation*} using $y^2=1$. Summing everything we can bound the Hessian of the expected regularised loss as it is continuous and we optimise over bounded $w$ and $Q$ has a bounded support.
    In particular, since $\forall x, s(x)(1-s(x) \leq \frac{1}{4} \leq 1$:
    \begin{equation*}
        ||\nabla^2 \Gamma_\lambda(w) || \leq || \frac{1}{2}Cov(Q_n) + \frac{1}{2}zz^T + 2\lambda ||
    \end{equation*}
    which can be taken as a smoothness constant.
    Then we can find an $\varepsilon^2$-approximated solution $\Tilde{w}$ of the minimiser $w^*$ in $O(2 \log(1/\varepsilon))$ such that 
    $||\Tilde{w}-w^*||=O(\varepsilon)$ and by Lipschitzness of the loss with respect to $w$ (thanks to Assumption~\ref{assum:B}), we get the result by plugging $\Tilde{w}$ into the loss. Computation of the gradient cost $O(nd)$ at each iteration of the gradient descent hence the result.
    \qed
\end{proof}


\subsection{Lipschitzness}\label{appendsec:Lip}

Proof of Proposition~\ref{prop:Lip}:

\begin{proof}
    To prove the Lipschitzness with respect to $z$, observe that only the negative part with $l_-$ changes. Considering some $z$ and $z'$ in $\cX$, for all $f\in \cH$:
    \begin{align*}
        |\bbE_{(x,y)\sim \LaDi} l(f(x),y) - \bbE_{(x,y)\sim P_{Q+|z'-}} l(f(x),y)| &= |\frac{1}{2} l_-(f(z)) - \frac{1}{2} l_-(f(z'))| \\
        &\leq \quad \frac{1}{2} L ||f(z)-f(z')|| \\
        &\leq \quad \frac{1}{2} LW ||z-z'||
    \end{align*}
    therefore we deduce that the two infima $D(z|Q)$ and $D(z'|Q)$ also verify:
    \begin{equation*}
        |D(z|Q) -D(z'|Q) | \leq \frac{1}{2} LW ||z-z'||.
    \end{equation*}

    For the Lipschitzness with respect to the distribution, on the contrary only the $l_+$ part changes. For two distributions $Q$ and $Q'$ of $\cP(\cX)$, for all $f\in \cH$:
        \begin{align*}
        |\bbE_{(x,y)\sim P_{Q+|z-}} l(f(x),y) - \bbE_{(x,y)\sim P_{Q'+|z-}} l(f(x),y)| &= |\frac{1}{2} \bbE_{x\sim Q}l_+(f(x)) - \frac{1}{2} \bbE_{x'\sim Q'}l_+(f(x'))|
    \end{align*}
    For any coupling $\pi$ of $Q$ and $Q'$: 
    \begin{align*}
        |\frac{1}{2} \bbE_{x\sim Q}l_+(f(x)) - \frac{1}{2} \bbE_{x'\sim Q'}l_+(f(x'))| &= |\frac{1}{2} \bbE_{(x,x')\sim \pi}[l_+(f(x)) - l_+(f(x'))]|\\
        &\leq \bbE_{(x,x')\sim \pi} \frac{1}{2} L ||f(x)-f(x')|| \\
        &\leq \bbE_{(x,x')\sim \pi} \frac{1}{2} LW ||x-x'||
    \end{align*}
    therefore, we deduce by minimising over the coupling that:
    \begin{equation*}
         \forall f\in \cH,|\bbE_{(x,y)\sim P_{Q+|z-}} l(f(x),y) - \bbE_{(x,y)\sim P_{Q'+|z-}} l(f(x),y)| \leq \frac{1}{2} LW(Q,Q')
    \end{equation*}
    and so
    \begin{equation*}
         |D(z|Q)-D(z|Q')| \leq \frac{1}{2} LW(Q,Q')
    \end{equation*} \qed
\end{proof}


\subsection{Convexity}\label{appendsec:convex}

Proof of Proposition~\ref{prop:convex}

\begin{proof}
    Assume by contradiction that some distinct points $z_1,z_2\in \cX$ have depths bigger than some $\alpha$, and some point $z_t = t z_1 + (1-t) z_2$ for some $t\in(0,1)$ has depth strictly smaller than $\alpha$. Then there must exist some $w^*$ such that:
    \begin{equation}\label{eq:bigineq}
        \lossD(z_t|Q) \leq \bbE_{(x,y)\sim P_{Q+|z_t-}}l(\langle w^*,x\ \rangle,y) < \alpha \leq \lossD(z_1|Q), \lossD(z_2|Q)
    \end{equation}
 (for instance, picking the $w^*$ that gives the infinum of the depth on $z_t$ or a value close to that infinum). But we also must have that, plugging the weight $w^*$ in the expected loss for $z_1$:
 \begin{equation}\label{eq:dz1}
     \lossD(z_1|Q) \leq \bbE_{(x,y)\sim P_{Q+|z_1-}}l(\langle w^*,x\ \rangle,y) 
 \end{equation} and similarly for $z_2$: 
 \begin{equation}\label{eq:dz2}
      \lossD(z_2|Q) \leq \bbE_{(x,y)\sim P_{Q+|z_2-}}l(\langle w^*,x\ \rangle,y).
 \end{equation}
 Assume that $\langle w^*, z_1 - z_2 \rangle = 0$.
 Then $l_-(\langle w^*, z_1\rangle) = l_-(\langle w^*, z_t\rangle) = l_-(\langle w^*, z_2\rangle)$ and so the expected losses using $w^*$ for distribution $P_{Q+|z_1-}, P_{Q+|z_t-}$ and $P_{Q+|z_2-}$ are also equal. But that is not possible, as eq.~\ref{eq:bigineq} states that such value should be strictly lower than $\alpha$ while eq.~\ref{eq:dz1} and \ref{eq:dz2} state that it should be higher or equal. So from now let us assume that $\langle w^*, z_1 - z_2 \rangle$ is not zero, and without loss of generality, assume instead $\langle w^*, z_1 - z_2 \rangle<0$ (the positive case is proven by symmetry). Assume also that $l_-$ is monotonely increasing (the decreasing case also works by symmetry).
 Notice that $z_1 = z_t + (1-t)(z_1-z_2)$ and $z_2 = z_t + t(z_2-z_1)$. Then by monotonicity: 
 \begin{align*}
    l_-(\langle w^*, z_1\rangle) &= l_-(\langle w^*, z_t\rangle + (1-t)\langle w^*, z_1-z_2\rangle) \\
    &\leq l_-(\langle w^*, z_t\rangle)
 \end{align*}
therefore $\bbE_{(x,y)\sim P_{Q+|z_1-}}l(\langle w^*,x\ \rangle,y) \leq \bbE_{(x,y)\sim P_{Q+|z_t-}}l(\langle w^*,x\ \rangle,y) \leq \alpha$, which is a contradiction with eq.~\ref{eq:dz1}. Switching the sign of $\langle w^*, z_1 - z_2 \rangle$ or the monotonicity of $l_-$ switch the role of $z_1$ and $z_2$, but at least one them will have its depth smaller than $\alpha$, hence the reductio ad absurdum. \qed
\end{proof}

The proof also works for the third variant of loss depth using 2-norm regularisation inside Def.~\ref{def:lossDepth}, where the regularisation is included into the total loss function, as adding a term $\lambda ||w^*||^2$ does not change the inequalities when evaluating on $z_1,z_2$ and $z_t$.


\subsection{Statistical convergence}\label{appendsec:statrates}

Proof of Lemma~\ref{lem:reggen}:

\begin{proof}
    The minimisers are respect to $F$ instead of $l$ so we want $F$ to appear:
    \begin{align*}
        \regD(z|Q_n)-\regD(z|Q) &= \bbE_{(x,y)\sim\LaDis{Q_n}{z}} l(\hat{f}^*_{reg}(x),y) - \bbE_{(x,y)\sim\LaDis{Q}{z}} l(f^*_{reg}(x),y)\\
        &= \bbE_{(x,y)\sim\LaDis{Q_n}{z}} F(\hat{f}^*_{reg}(x),y) - \bbE_{(x,y)\sim\LaDis{Q}{z}} F(f^*_{reg}(x),y)\\ &+ reg(f^*_{reg}) - reg(\hat{f}^*_{reg}) . 
    \end{align*}
    Generalisation error of $F$ is the same as for $l$, and the rest of the proof is similar to Lemma~\ref{lemma:gen}.

\end{proof}

Proof of Theorem~\ref{prop:rates}:

\begin{proof}
Note that for $\varphi$ feature map of a RKHS, by the reproducing property $w(\cdot) = \langle w,\varphi(\cdot)\rangle$.
For ease of notation, since Assumption~\ref{assum:W} is satisfied, let us therefore identify the ball $\{w \in \cF \big| ||w|| \leq W \}$ and the space of functions $\{f:x \mapsto \langle w, \varphi(x)\rangle\}$ and denote it $F_W$ (for plain features in $\bbR^d$ we identify by isomorphism $w\in\bbR^d$ and its dual $w^*$).
    Let $C$ be such that $\forall x \in supp(Q), w\in B_W, ||l_+(\langle w, \varphi(x)\rangle,y)|||\leq C$ (it exists for the hinge loss and logistic loss thanks to Assumptions~\ref{assum:B} and \ref{assum:W} by continuity). Then we have the bound for all $w\in F_W$ (see, e.g., \citet{Liu21} for ready-to-apply version of the theorem or ~\citet{BartlettM02} for the original result) with probability greater than $1-\delta$ for any $\delta>0$:
    \begin{align*}
        |\bbE_{(x,y) \sim \LaDis{Q_n}{z}}l_+(\langle w, \varphi(x)\rangle,y)  -  \bbE_{(x,y) \sim \LaDi}l_+(\langle w, \varphi(x)\rangle,y)| 
        &\leq 2\cR_{Q,n}(\cF_W) + 2C\sqrt{\frac{log(2/\delta)}{n}}
    \end{align*}
    where we used $\cR_{Q,n}(l_+ \circ \cF_W) \leq \cR_{Q,n}(\cF_W)$ by 1-Lipschitzness of the hinge and logistic losses (see for instance \citet{Maurer16}).
    According to Theorem 3 of~\citet{kakade2008complexity}, with Assumption~\ref{assum:B} and \ref{assum:W} using the norm as regularisation, we have that:
    \begin{equation*}
        \cR_{Q,n}(\cH) = WB\sqrt{\frac{1}{\lambda n}}
    \end{equation*}
    This concludes for $LRD$ using Lemma~\ref{lemma:gen}. To conclude for $SVM\regD$ using Lemma~\ref{lem:reggen}, we need also to bound the difference of regularisation terms.
    Let us denote :
    \begin{equation}
        \hat{w}_\lambda^* = argmin_w \mathbb{E}_{(x,y) \sim\LaDis{Q_n}{z}} F_\lambda(w,x,y)
    \end{equation}
    and
    \begin{equation}
        w_\lambda^* = argmin_w \mathbb{E}_{(x,y) \sim\LaDis{Q}{z}} F_\lambda(w,x,y)
    \end{equation}
The term of difference of norms is easier reframed as a norm of difference, thanks to Assumption~\ref{assum:W}:
    \begin{align}
    |\lambda(||\hat{w}_\lambda^*||^2-||\hat{w}_\lambda^*||^2)| &= \lambda |\langle \hat{w}^*_\lambda+w^*_\lambda,\hat{w}_\lambda-w^*_\lambda\rangle|\notag \\
    &\leq \lambda(||\hat{w}^*_\lambda ||+||w^*_\lambda ||)||\hat{w}^*_\lambda-w^*_\lambda||\notag \\
    &\leq 2\lambda W||\hat{w}_\lambda-w^*_\lambda|| \label{eq:boundnormdiff}
    \end{align}
    We conclude thanks to a result by~\citet{SteinwartChristmann2008}, p.225 in the proof of their Theorem 6.24 that gives, with probability greater than $1-e^{-\tau}$:
    \begin{equation*}
        Q^n(||\hat{w}_\lambda-w^*_\lambda|| < \frac{L}{\lambda}(\sqrt{\frac{2\tau}{n}}+\sqrt{\frac{1}{n}}+\frac{4\tau}{3n}))
    \end{equation*}
    (where here the Lipschitz constant $L=1$). 
        Indeed, as in their proof, thanks to their Corollary 5.10, there exists a function $h$ such that:
    \begin{align*}
        ||\hat{w}_\lambda-w^*_\lambda|| &\leq \frac{1}{\lambda} ||\mathbb{E}_{(x,y) \sim\LaDis{Q_n}{z}} h(x,y)\varphi(x) -\mathbb{E}_{(x,y) \sim\LaDis{Q}{z}} h(x,y)\varphi(x)|| \\
        &= \frac{1}{\lambda} ||\mathbb{E}_{(x,y) \sim Q_n} h(x,y)\varphi(x) -\mathbb{E}_{(x,y) \sim Q} h(x,y)\varphi(x)||
    \end{align*}
    that verifies the concentration inequality via their Corollary 6.15:
    \begin{equation*}
        Q^n(||\mathbb{E}_{(x,y) \sim Q_n} h(x,y)\varphi(x) -\mathbb{E}_{(x,y) \sim Q} h(x,y)\varphi(x)|| < L(\sqrt{\frac{2\tau}{n}}+\sqrt{\frac{1}{n}}+\frac{4\tau}{3n})) \leq e^{-\tau}.
    \end{equation*} 
    \qed
\end{proof}


\section{Experiments}\label{appendsec:exp}
Computation were carried on a Macbook Pro 2020 with processor 2,3 GHz Intel Core i7 (4 cores) and memory 16 Go 3733 MHz LPDDR4X (graphic card is Intel Iris Plus Graphics 1536 Mo).

\subsection{Ranking correlations with densities}\label{appendsec:ranking}

In Figure~\ref{fig:rank}, we display the boxplots over 10 runs of the rank correlation coefficients (the \citet{kendall1938new} tau and the one of \citet{spearman1987proof}) between the ranks obtained by comparing densities and the ranks using data depths for a dataset with two Gaussian clusters of 100 samples each (total of $n=200$ samples) for dimension $d=2,4,6,8$, as described in section~\ref{subsec:illu}. The two kernel methods seem to come out on top, though one of their advantage is that their parameters have been automatically scaled with the data.  

    \begin{figure}[!hb]
        \centering
        \subfloat[Kendall's $\tau$ rank correlation coefficient]{%
            \includegraphics[width=.99\linewidth]{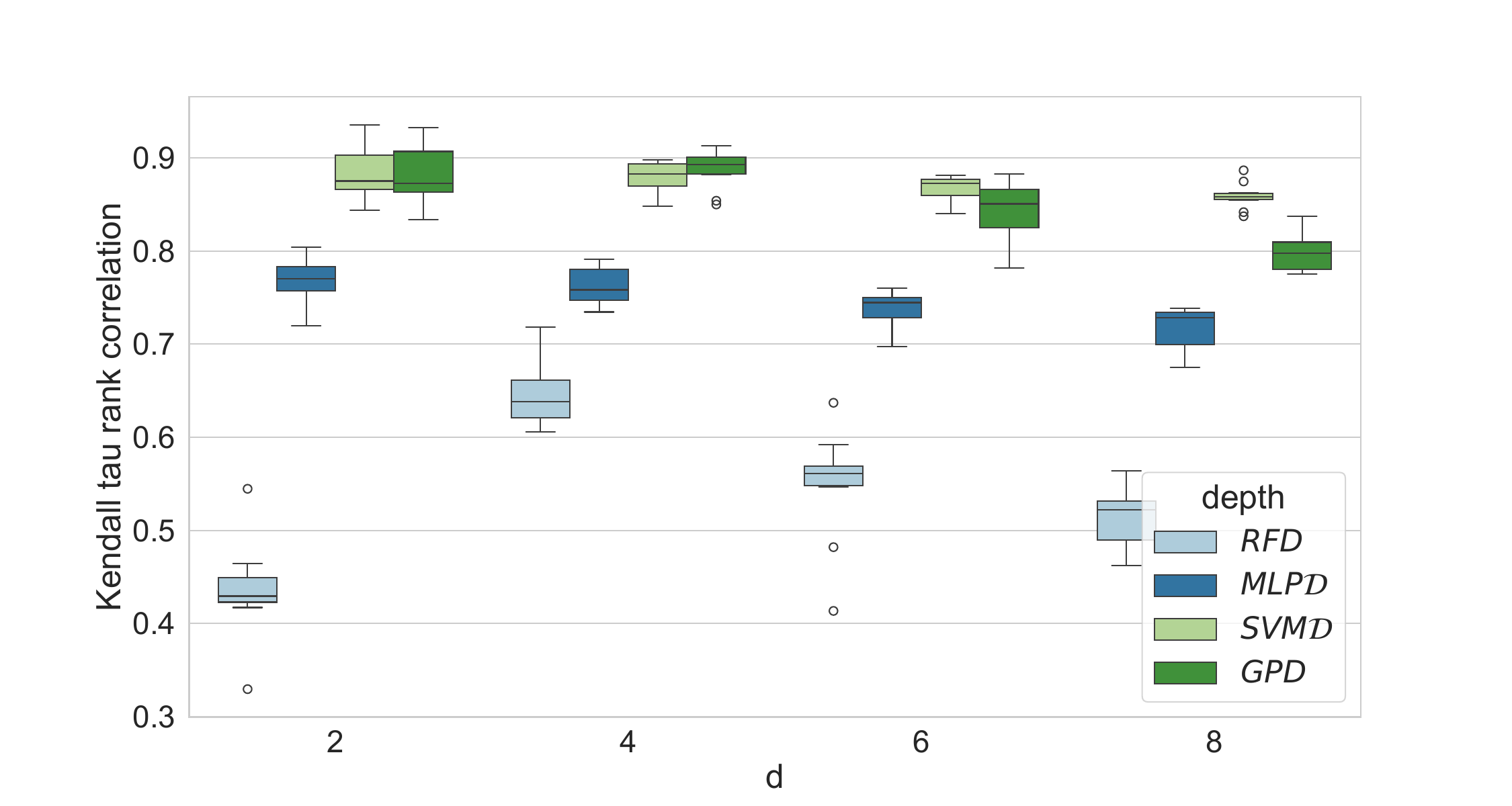}%
            \label{subfig:kendall}%
        }\\
        \subfloat[Spearman's rank correlation coefficient]{%
            \includegraphics[width=.99\linewidth]{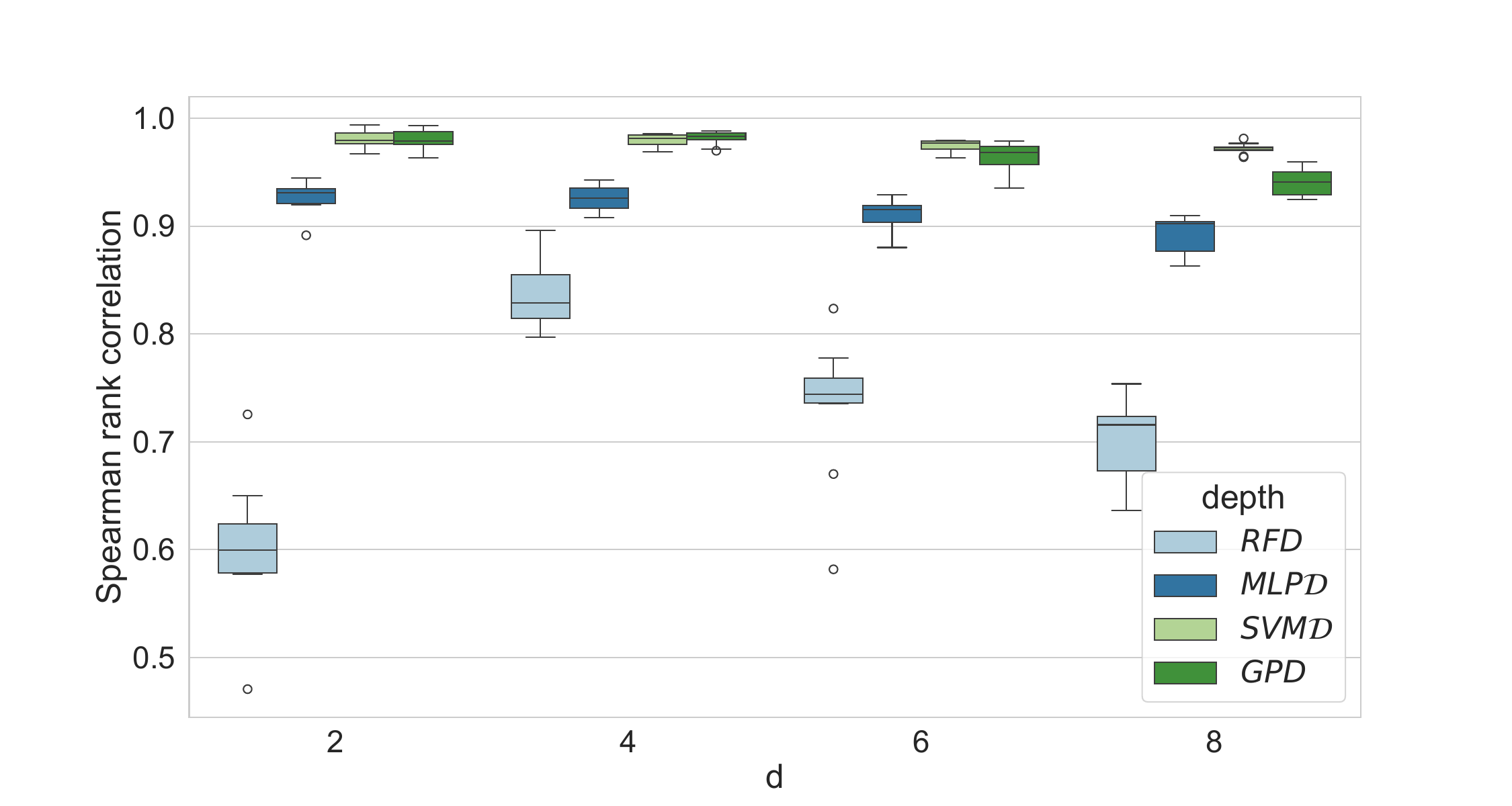}%
            \label{subfig:spearman}%
        }
        \caption{Rank correlations w.r.t. the true density for a bigaussian distribution (n = 200)}
        \label{fig:rank}
    \end{figure}


\subsection{MNIST and Fashion-MNIST}\label{appendsec:mnist}

\paragraph{Tables}
For a given dataset, the set $\cH$ of classifiers should not be too simple but neither necessarily too complex in order to avoid overfitting. Here with some anomaly detection experiments on MNIST~\citep{mnist} and Fashion-MNIST~\citep{xiao2017/online}, we show that for this datasets deep architectures are not mandatory to achieve fair results (it was also already confirmed by OC-SVM before) if you do not need to memorise the structure of your dataset (as it is done in the weights of neural networks). 
The task consists in taking one class as the inlier class and all the other classes as anomalies, then we measure the Area Under the Curve of the Receiving Operator Curve (AUC-ROC) in Tables~\ref{tab:mnist} (MNIST) and~\ref{tab:fmnist} (Fashion-MNIST), while competitors are shown in Tables~\ref{tab:mnist-compet} and~\ref{tab:fmnist-compet} (all results are round up to 2 decimals, AUC of 1.00 are actually between 0.995 and 0.999). To compare with methods that use a training set, we split the data, keeping 80\% as training set and 20\% as test set.
We apply the two depths studied in section~\ref{sec:linear}: $LR\regD$ with logistic regression and $SVM\regD$ (both with $\lambda=1$), which have both optimisation and statistical guarantees as mentioned previously.
$LR\regD$ has slightly the upper-hand compared to $SVM\regD$, it can almost perfectly classify some classes (class 1 of MNIST and Fashion-MNIST for instance) and it compares as good or better than some of the deep architectures which use sometimes several convolutional layers such as GAN~\citep{IPMI}, DSVDD~\citep{ruff18a}, DSEBM \citep{zhai16}, DAGMM \citep{Zong18}, GEOM- \citep{golan18} and variational auto-encoders~\citep{VAE}, even though logistic regression is just one neuron with sigmoid activation. This is thanks to the fact that the optimisation algorithm is run for each point instead of memorising. This has of course an impact on speed of computation, however our results shows that even taking only $n=1000$ or even $100$ samples of the training set can lead to good performance (sometimes better than running on the full training set, but it may be because we did only one run for the full training set and average over 10 runs for smaller training sets to take into account time of computation constraints).

\begin{table}[htb]
  \caption{AUC-ROC on MNIST (rounded to 2 decimals)}
  \label{tab:mnist}
  \resizebox{1.0\linewidth}{!}{
  \centering
  \begin{tabular}{lrrrrrrrrrrr}
    \toprule
method/class & 0 & 1 & 2 & 3 & 4 & 5 & 6 & 7 & 8 & 9 & \textbf{MEAN} \\
    \midrule
$SVM\mathcal{D}$ ($n=100$) & 0.96 & 0.99 & 0.87 & 0.88 & 0.92 & 0.91 & 0.97 & 0.94 & 0.85 & 0.94 & \textbf{0.92} \\
$SVM\mathcal{D}$ ($n=1000$) & 0.98 & 0.99 & 0.91 & 0.92 & 0.94 & 0.94 & 0.98 & 0.95 & 0.88 & 0.96 & \textbf{0.94} \\
$SVM\mathcal{D}$ & 0.96 & 0.99 & 0.87 & 0.87 & 0.92 & 0.91 & 0.96 & 0.94 & 0.84 & 0.94 & \textbf{0.92} \\
$LR\mathcal{D}$ ($n=100$) & 0.99 & 1.00 & 0.89 & 0.92 & 0.94 & 0.91 & 0.98 & 0.95 & 0.87 & 0.94 & \textbf{0.94} \\
$LR\mathcal{D}$ ($n=1000$) & 0.99 & 1.00 & 0.92 & 0.94 & 0.96 & 0.95 & 0.99 & 0.97 & 0.89 & 0.96 & \textbf{0.96} \\
$LR\mathcal{D}$ & 1.00 & 1.00 & 0.93 & 0.95 & 0.97 & 0.96 & 0.99 & 0.97 & 0.91 & 0.97 & \textbf{0.97} \\
    \bottomrule
  \end{tabular}}
\end{table}

\begin{table}[ht]
  \caption{AUC-ROC on Fashion-MNIST (rounded to 2 decimals)}
  \label{tab:fmnist}
  \resizebox{1.0\linewidth}{!}{
  \centering
  \begin{tabular}{lrrrrrrrrrrr}
    \toprule
method/class & 0 & 1 & 2 & 3 & 4 & 5 & 6 & 7 & 8 & 9 & \textbf{MEAN} \\
    \midrule
$SVM\mathcal{D}$ ($n=100$) & 0.88 & 0.97 & 0.86 & 0.90 & 0.87 & 0.86 & 0.81 & 0.95 & 0.81 & 0.94 & \textbf{0.89} \\
$SVM\mathcal{D}$ ($n=1000$) & 0.88 & 0.98 & 0.86 & 0.90 & 0.89 & 0.86 & 0.80 & 0.94 & 0.79 & 0.94 & \textbf{0.89} \\
$SVM\mathcal{D}$ & 0.86 & 0.98 & 0.85 & 0.88 & 0.86 & 0.86 & 0.79 & 0.92 & 0.79 & 0.93 & \textbf{0.87} \\
$LR\mathcal{D}$ ($n=100$) & 0.91 & 0.98 & 0.89 & 0.93 & 0.91 & 0.89 & 0.84 & 0.98 & 0.86 & 0.97 & \textbf{0.92} \\
$LR\mathcal{D}$ ($n=1000$) & 0.91 & 0.99 & 0.89 & 0.94 & 0.91 & 0.90 & 0.84 & 0.98 & 0.86 & 0.98 & \textbf{0.92} \\
$LR\mathcal{D}$ & 0.91 & 0.99 & 0.90 & 0.93 & 0.90 & 0.91 & 0.84 & 0.99 & 0.85 & 0.98 & \textbf{0.92} \\
    \bottomrule
  \end{tabular}}
\end{table}

\begin{table}[ht]
	\centering
	\caption{AUC-ROC for MNIST dataset with competitors taken from \cite{Wu21}}
	\label{tab:mnist-compet}
	\resizebox{1.0\linewidth}{!}{
	\begin{tabular}{llllllllllll}
\toprule
method/class & 0     & 1     & 2     & 3     & 4     & 5     & 6     & 7     & 8     & 9     & \textbf{MEAN}    \\ \midrule 
OCSVM \citep{ocsvm}  & 0.99 & 1.00 & 0.90 & 0.95 & 0.96 & 0.97 & 0.98 & 0.97 & 0.85 & 0.96 & \textbf{0.95} \\ 
KDE \citep{Bishop07}    & 0.89 & 1.00 & 0.71 & 0.69 & 0.84 & 0.78 & 0.86 & 0.88 & 0.67 & 0.83 & \textbf{0.81} \\ 
DAE \citep{DAE}    & 0.89 & 1.00 & 0.79 & 0.85 & 0.89 & 0.82 & 0.94 & 0.92 & 0.74 & 0.92 & \textbf{0.88} \\ 
VAE \citep{VAE}    & 1.00 & 1.00 & 0.94 & 0.96 & 0.97 & 0.96 & 0.99 & 0.98 & 0.92 & 0.98 & \textbf{0.97} \\ 
Pix CNN \citep{PixCNN} & 0.53 & 1.00 & 0.48 & 0.52 & 0.74 & 0.54 & 0.59 & 0.79 & 0.34 & 0.66 & \textbf{0.62} \\ 
GAN \citep{IPMI}     & 0.93 & 1.00 & 0.81 & 0.82 & 0.82 & 0.80 & 0.89 & 0.90 & 0.82 & 0.89 & \textbf{0.87} \\ 
AND \citep{AND}     & 0.98 & 1.00 & 0.95 & 0.95 & 0.96 & 0.97 & 0.99 & 0.97 & 0.92 & 0.98 & \textbf{0.97} \\ 
AnoGAN \citep{IPMI} & 0.97 & 0.99 & 0.85 & 0.89 & 0.89 & 0.88 & 0.95 & 0.94 & 0.85 & 0.92 & \textbf{0.91} \\ 
DSVDD  \citep{ruff18a}  & 0.98 & 1.00 & 0.92 & 0.92 & 0.95 & 0.89 & 0.98 & 0.95 & 0.94 & 0.97 & \textbf{0.95} \\ 
OCGAN \citep{ocgan}   & 1.00 & 1.00 & 0.94 & 0.96 & 0.98 & 0.98 & 0.99 & 0.98 & 0.94 & 0.98 & \textbf{0.98} \\
OCADNN \citep{Wu21} & 1.00 & 1.00 & 0.96 & 0.97 & 0.97 & 0.98 & 0.99 & 0.98 & 0.95 & 0.98 & \textbf{0.98} \\ 
\bottomrule
				\end{tabular}}
\end{table}

\begin{table}[!ht]
	\centering
	\caption{AUC-ROC for Fashion-MNIST dataset with competitors taken from \cite{Wu21}}
	\label{tab:fmnist-compet}
	\begin{tabular}{lr}
\toprule
method/class &  \textbf{MEAN AUC}    \\ \midrule 
DSVDD \citep{ruff18a} & \textbf{0.93} \\
DAGMM \citep{Zong18} & \textbf{0.52} \\
DSEBM \citep{zhai16} & \textbf{0.87} \\
GEOM- \citep{golan18} & \textbf{0.80} \\
GEOM \citep{golan18} & \textbf{0.94} \\
GOAD \citep{BergmanH20} & \textbf{0.94} \\
OCADNN \cite{Wu21} & \textbf{0.95} \\
\bottomrule
\end{tabular}
\end{table}

    
\end{document}